\documentclass{article}
\usepackage{ijcai16}
\usepackage[para,online,flushleft]{threeparttable}
\usepackage[]{graphicx}
\usepackage{balance} 
\usepackage{booktabs}
\usepackage{etoolbox}
\usepackage{tabularx}
\usepackage{amsmath}
\usepackage{leftidx}
\usepackage{bm}
\usepackage{comment}
\usepackage{amsthm}
\usepackage{amssymb}
\usepackage{times}
\usepackage{enumitem}
\usepackage{algorithm}
\usepackage{algpseudocode}
\usepackage{subfigure}
\newtheorem{definition}{Definition}
\newtheorem{theorem}{Theorem}[section]
\newtheorem{remark}{Remark}
\newtheorem{lemma}[theorem]{Lemma}

\title{Technical Report: A Generalized Matching Pursuit Approach for Graph-Structured Sparsity}
\author{Feng Chen, Baojian Zhou \\ 
Computer Science Department, University at Albany -- SUNY \\
1400 Washington Avenue, Albany, NY, USA\\
\{fchen5, bzhou6\}@albany.edu}

\begin{document}
\maketitle

\begin{abstract}

Sparsity-constrained optimization is an important and challenging problem that has wide applicability in data mining, machine learning, and statistics.  In this paper, we focus on sparsity-constrained optimization in cases where the cost function is a general nonlinear function and, in particular, the sparsity constraint is defined by a \textbf{graph-structured sparsity} model. Existing methods explore this problem in the context of sparse estimation in linear models. To the best of our knowledge, this is the first work to present an efficient approximation algorithm, namely, \textsc{Graph}-structured \textsc{M}atching \textsc{P}ursuit (\textsc{Graph}-\textsc{Mp}), to optimize a general nonlinear function subject to graph-structured constraints. We prove that our algorithm enjoys the strong guarantees analogous to those designed for linear models in terms of convergence rate and approximation accuracy. As a case study, we specialize \textsc{Graph}-\textsc{Mp} to optimize a number of well-known graph scan statistic models for the connected subgraph detection task, and empirical evidence demonstrates that our general algorithm performs superior over state-of-the-art methods that are designed specifically for the task of connected subgraph detection.  
% on both running time and accuracy.  

\end{abstract}

%Sparsity-constrained optimization is an important and challenging problem that has wide applicability in data mining, machine learning, and statistics.  In this paper, we focus on nonlinear optimization problems that involve, in particular, \textbf{graph-structured sparsity} constraints. Existing methods explore this problem in the context of sparsity estimation in linear models. To the best of our knowledge, this is the first work to present nearly-linear time approximation algorithms to optimize nonlinear objective functions of arbitrary forms. We prove that under mild conditions our methods enjoy the strong guarantees analogous to those designed for linear models in terms of  convergence rate and variable estimation accuracy. Extensive empirical evidences demonstrate that our methods are superior to the state-of-the-art methods in spatial anomalous cluster detection and sparse principle component analysis  tasks. 

\section{Introduction}

In recent years, that is a growing demand on efficient computational methods for analyzing high-dimensional data  in a variety of applications such as bioinformatics, medical
imaging, social networks, and astronomy.  In many settings, sparsity has been shown effective to model latent structure in high-dimensional data and at the same time remain a mathematically tractable concept. Beyond the ordinary, extensively studied, sparsity model, a variety of \textbf{structured sparsity models} have been proposed in the literature, such as the sparsity models defined through trees~\cite{hegde2014fast}, groups~\cite{jacob2009group}, clusters~\cite{huang2011learning}, paths~\cite{asterisstay2015icml}, and connected subgraphs~\cite{hegde2015nearly}. 
These sparsity models are designed to capture the interdependence of the locations of the non-zero components via prior knowledge, and are considered in the general sparsity-constrained optimization problem:
\begin{eqnarray}
\min_{{\bf x} \in \mathbb{R}^n} f({\bf x})\ \ s.t. \ \ \text{supp}({\bf x}) \in \mathbb{M}, \label{problem:general}
\end{eqnarray}
where $f: \mathbb{R}^n \rightarrow \mathbb{R}$ is a differentiable cost function and the sparsity model $\mathbb{M}$ is defined as a family of structured supports: $\mathbb{M} = \{S_1, S_2, \cdots, S_L\}$, where $S_i \subseteq [n]$ satisfies a certain structure property (e.g., trees, groups, clusters). The original $k$-sparse recovery problem corresponds to the particular case where the model $\mathbb{M} = \{S\subseteq [n]\ |\ |S| \le k\}$. 

%A number of methods focus on structured sparsity in the context of sparse recovery problems such as compressive sensing and sparse linear models, which have been  extensively reviewed in~\cite{duarte2011structured,hegde2015fast}. These method are designed specifically to  optimize the least square cost function $f({\bf x}) = \|{\bf y} - {\bf A} {\bf x}\|_2^2$, where ${\bf y}\in \mathbb{R}^m$, and ${\bf A} \in \mathbb{R}^{m\times n}$  is often assumed to obey a good property known as the \textit{Restricted Isometry Property} (RIP).

The methods that focus on general nonlinear cost functions fall into two major categories, including \textbf{structured sparsity-inducing norms based} and \textbf{model-projection based}, both of which often assume that the cost function $f({\bf x})$ satisfies a certain convexity/smoothness condition, such as \textit{Restricted Strong Convexity/Smoothness} (RSC/RSS) or \textit{Stable Mode-Restricted Hessian} (SMRH). In particular, the methods in the first category replace the structured sparsity model with regularizations by a sparsity-inducing norm that is typically non-smooth and non-Euclidean~\cite{bach2012structured}. The methods in the second category decompose Problem~(\ref{problem:general}) into an unconstrained subproblem and a model projection oracle that finds the best approximation of an arbitrary ${\bf x}$ in the model $\mathbb{M}$: \[\text{P}({\bf x}) = \arg \min_{{\bf x}^\prime \in \mathbb{R}^n} \|{\bf x} - {\bf x}^\prime\|_2^2\ \ \ s.t. \ \ \ \text{supp}({\bf x}^\prime) \in \mathbb{M}.\] A number of methods are proposed specifically for the k-sparsity model $\mathbb{M} = \{S\subseteq [n]\ |\ |S| \le k\}$, including the forward-backward algorithm~\cite{zhang2009adaptive}, the gradient descent algorithm~\cite{tewari2011greedy},  the gradient hard-thresholding algorithms~\cite{yuan2014icml,bahmani2013greedy,jain2014iterative}, and the Newton greedy pursuit algorithm~\cite{yuan2014newton}. A limited number of methods are proposed for other types of structured sparsity models via projected gradient descent, such as the union of subspaces~\cite{blumensath2013compressed} and the union of nested subsets~\cite{bahmani2016learning}. 

In this paper, we focus on general nonlinear optimization subject to graph-structured sparsity constraints. Our approach applies to data with an underlying graph structure in which nodes corresponding to $\text{supp}({\bf x})$ form a small number of connected components. By a proper choice of the underlying graph, several other structured sparsity models such as the ``standard'' $k$-sparsity, block sparsity, cluster sparsity, and tree sparsity can be encoded as special cases of graph-structured sparsity~\cite{hegde2015fast}. 

We have two key observations: 1) \textbf{Sparsity-inducing norms.} There is no known sparsity-inducing norm that is able to capture graph-structured sparsity. The most relevant norm is generalized fused lasso~\cite{xin2014efficient} that enforces the smoothness between neighboring entries in ${\bf x}$, but does not have  fine-grained control over the number of connected components. Hence, existing methods based on sparsity-inducing norms are not directly applicable to the problem to be optimized.  2) \textbf{Model projection oracle.} There is  no exact model projection oracle for a graph-structured sparsity model, as this exact projection problem is NP-hard due to a reduction from the classical Steiner tree problem~\cite{hegde2015nearly}. As most existing model-projection based methods assume an exact model projection oracle, they are not directly applicable here as well. To the best of our knowledge, there is only one recent approach that admits inexact projections for a graph-structured sparsity model by assuming ``head'' and ``tail''  approximations for the projections, but is only applicable to linear regression problems~\cite{hegde2015nearly}. This paper will generalize this approach to optimize general nonlinear functions. The main contributions of our study are summarized as follows:
\begin{itemize}
\item \textbf{Design of an efficient approximation algorithm.} A new and efficient algorithm, namely, \textsc{Graph}-\textsc{Mp}, is developed to approximately solve Problem~(\ref{problem:general}) with a differentiable cost function and a graph-structured sparsity model. We show that \textsc{Graph}-\textsc{Mp} reduces to a state-of-the-art algorithm for graph-structured compressive sensing and linear models, namely, \textsc{Graph}-\textsc{Cosamp}, when $f({\bf x})$ is  a least square loss function. 
\item \textbf{Theoretical analysis and connections to existing methods.} The convergence rate and accuracy of the proposed \textsc{Graph}-\textsc{Mp} are analyzed under a condition of $f({\bf x})$ that is weaker than popular conditions such as RSC/RSS and SMRH. We demonstrate that  \textsc{Graph}-\textsc{Mp} enjoy strong guarantees analogous to \textsc{Graph}-\textsc{Cosamp} on both convergence rate and accuracy. 
%\item \textbf{Develop of approximation algorithms for graph scan statistic models.}
\item \textbf{Compressive experiments to validate the effectiveness and efficiency of the proposed techniques.} The proposed \textsc{Graph}-\textsc{Mp} is applied to optimize a variety of graph scan statistic models for the task of connected subgraph detection. Extensive experiments demonstrate that \textsc{Graph}-\textsc{Mp} performs superior over state-of-the-art methods that are customized for the task of connected subgraph detection on both running time and accuracy. 
\end{itemize}

The rest of this paper is organized as follows. Section 2 introduces the graph-structured sparsity model. Section 3 formalizes the problem and presents an efficient algorithm \textsc{Graph}-\textsc{Mp}. Sections 4 and 5 present theoretical analysis. Section 6 gives the applications of \textsc{Graph}-\textsc{Mp}. Experiments are presented in Section 7, and Section 8 describes future work. 

%  and  proposes nonparametric scan statistics for heterogeneous graphs. Experiments on real Twitter datasets are presented in Section 4, and Section 5 describes future work.

%We present nearly-linear time approximation algorithms 

%As a case study, we consider graph scan statistic models 

%Graph-structured sparsity models have a variety of applications, such as ... 

%\vspace{-2mm}
\section{Graph-Structured Sparsity Model}
Given an underlying graph $\mathbb{G} = (\mathbb{V}, \mathbb{E})$ defined on the coefficients of the unknown vector ${\bf x}$, where $\mathbb{V} = [n]$ and $\mathbb{E} \subseteq \mathbb{V}\times \mathbb{V}$, a graph-structured sparsity model has the form:
\begin{small}
\begin{eqnarray}
\mathbb{M}(k, g) = \{S \subseteq \mathbb{V}\ |\ |S| \le k, \gamma(S) = g\},
\end{eqnarray}
\end{small}
\noindent where $k$ refers to an upper bound of the sparsity (total number of nodes) of $S$ and $\gamma(S) = g$ refers to the maximum number of connected components formed by the forest induced by $S$: \begin{small}$\mathbb{G}_{S} = (S, \mathbb{E}_S)$\end{small}, where \begin{small}$\mathbb{E}_S = \{(i, j) \ |\ i, j \in S, (i, j) \in \mathbb{E}\}$\end{small}. The corresponding model projection oracle is defined as 
\begin{eqnarray}
\text{P}({\bf x}) = \arg\min_{{\bf x}^\prime \in \mathbb{R}^n}  \|{\bf x} - {\bf x}^\prime\|_2^2\ \ s.t.\ \ \text{supp}({\bf x}^\prime) \in \mathbb{M}(k,g). \label{eqn:projection}
\end{eqnarray}
Solving Problem~(\ref{eqn:projection}) exactly is NP-hard due to a reduction from the classical Steiner tree problem. Instead of solving (\ref{eqn:projection}) exactly, two nearly-linear time approximation algorithms with the following complementary approximation guarantees are proposed in~\cite{hegde2015nearly}: 
\begin{itemize}[leftmargin=*]
\item \textbf{Tail approximation} ($\text{T}({\bf x})$):  Find \begin{small}$S\in \mathbb{M}(k_T, g)$\end{small} such that 
\begin{small}
\begin{eqnarray}
\|{\bf x} - {\bf x}_S\|_2 \le c_T \cdot \min_{S^\prime \in \mathbb{M}(k, g)} \|{\bf x} - {\bf x}_{S^\prime}\|_2,
\end{eqnarray}
\end{small}
\noindent where $c_T = \sqrt{7}$ and $k_T=5k$. 
\item \textbf{Head approximation} ($\text{H}({\bf x})$): Find \begin{small}$S\in \mathbb{M}(k_H,g)$\end{small} such that 
\begin{small}
\begin{eqnarray}
\|{\bf x}_S\|_2 \ge c_H\cdot \max_{S^\prime \in \mathbb{M}(k, g)} \|{\bf x}_{S^\prime}\|_2,
\end{eqnarray}
\end{small}

\vspace{-4mm}
\noindent where $c_H = \sqrt{1/14}$ and $k_H = 2k$.
\end{itemize}
If $c_T = c_H = 1$, then $\text{T}({\bf x}) = \text{H}({\bf x}) = S$ provides the exact solution of the model projection oracle: $\text{P}({\bf x}) = {\bf x}_S$, which indicates that the approximations stem from the fact that $c_T > 1$ and $c_H < 1$. We note that these two approximations originally involve additional budgets ($B$) based on edge weights, which are ignored in this paper by setting unit edge weights and $B = k - g$. 

\vspace{1mm}
\noindent \textbf{Generalization:} The above graph-structured sparsity model is defined based on the number of connected components in the forest induced by $S$. This model can be generalized to graph-structured sparsity models that are defined based on other graph topology constraints, such as density, k-core, radius, cut, and various others, as long as their corresponding head and tail approximations are available.

%test

%\vspace{-2mm}
\section{Problem Statement and Algorithm}

%A graph-structured sparsity-constrained problem is formulated as
Given the graph-structured sparsity model, $\mathbb{M}(k, g)$, as defined above, the sparsity-constrained optimization problem to be studied is formulated as:
\begin{small}
\begin{eqnarray}
\min_{{\bf x} \in \mathbb{R}^n} f({\bf x})\ \ s.t. \ \ \text{supp}({\bf x}) \in \mathbb{M}(k, g),
\end{eqnarray}
\end{small}
\noindent where $f: \mathbb{R}^n \rightarrow \mathbb{R}$ is a differentiable cost function, and the upper bound of sparsity $k$ and the maximum number of connected components $g$ are predefined by users. 

Hegde et al. propose \textsc{Graph}-\textsc{Cosamp}, a variant of \textsc{Cosamp}~\cite{hegde2015nearly} to optimize the least square cost function $f({\bf x}) = \|{\bf y}-{\bf A}{\bf x}\|_2^2$ based on the head and tail approximations. The authors show that \textsc{Graph}-\textsc{Cosamp} achieves an
information-theoretically optimal sample complexity
for a wide range of parameters.
In this paper, we genearlize \textsc{Graph}-\textsc{Cosamp} and propose a new algorighm named as \textsc{Graph}-\textsc{Mp} for Problem (6), 
%The basic steps of \textsc{Graph}-\textsc{Mp} are 
as shown in Algorithm~\ref{alg:Graph-MP}. The first step (Line 3) in each iteration, ${\bf g} = \nabla f({\bf x}^i)$,
evaluates the gradient of the cost function at the current estimate. Then a subset of nodes are identified via head
approximation, $\Gamma = \text{H}({\bf g})$, that returns a support set with head value at least a constant fraction of the optimal head value, in which pursuing the minimization
will be most effective. This subset is then merged with the support of the current estimate to obtain the
merged subset $\Omega$, over which the function f is minimized to produce an intermediate estimate, ${\bf b} = \arg\min_{{\bf x} \in \mathbb{R}^n} f({\bf x})\ \ s.t. \ \ {\bf x}_{\Omega^c} = 0$. Then a subset of nodes are identified via tail
approximation, $B = \text{T}({\bf b})$, that returns a support set with tail value at most a constant times larger than the optimal tail value. The iterations terminate when the halting condition holds. There are two popular options to define the halting condition: 1) the change of the cost function from the previous iteration is less than a threshold ($|f({\bf x}^{i+1}) - f({\bf x}^i)| \le \epsilon$); and 2) the change of the estimated minimum from the previous iteration is less than a threshold ($\|{\bf x}^{i+1} - {\bf x}^i\|_2 \le \epsilon$), where $\epsilon$ is a predefined threshold (e.g., $\epsilon = 0.001$). 
%The iterations terminate when the
%change of the cost function or that of the estimated minimum from the previous iteration is less than a threshold. 

\begin{algorithm}
\caption{\textsc{Graph}-\textsc{Mp}}
\label{algorithm1}
\begin{algorithmic}[1]
\State $i = 0$, ${\bf x}^i = 0$;
\Repeat
\State ${\bf g} = \nabla f({\bf x}^i)$;
\State $\Gamma = \text{H}({\bf g})$;
\State $\Omega = \Gamma \cup \text{supp}({\bf x}^i)$
\State ${\bf b} = \arg\min_{{\bf x} \in \mathbb{R}^n} f({\bf x})\ \ s.t. \ \ {\bf x}_{\Omega^c} = 0$
\State $B = \text{T}({\bf b})$;
\State ${\bf x}^{i+1} = {\bf b}_B$
\Until{halting condition holds}
\State {\bf return} ${\bf x}^{i+1}$
\end{algorithmic}\label{alg:Graph-MP}
\end{algorithm}
%\vspace{-5mm}

\section{Theoretical Analysis of \textsc{Graph}-\textsc{Mp} under SRL condition}
In this section, we give the definition of Stable Restricted Linearization (SRL)~\cite{bahmani2013greedy} and we show that our \textsc{Graph-Mp} algorithm enjoys a theorectial approximation guarantee under this SRL condition.

\begin{definition}[Restricted Bregman Divergence~\cite{bahmani2013greedy}]
We denote the restricted Bregman divergence of $f$ as $B_f \Big(\cdot \| \cdot \Big)$. The restricted Bregman divergence of $f: \mathbb{R}^p \rightarrow \mathbb{R}$ between points ${\bf x}$ and ${\bf y}$ is defined as
\begin{equation}
{B}_{f} \Big( {\bf x} \| {\bf y} \Big) = f({\bf x}) - f({\bf y}) - \langle \nabla_f({\bf y}), {\bf x} - {\bf y} \rangle,
\end{equation}
where $\nabla_f({\bf y})$ gives a restricted subgradient of $f$. We say vector $\nabla f({\bf x})$ is a restricted subgradient of $f: \mathbb{R}^p \rightarrow \mathbb{R}$ at point ${\bf x}$ if
\begin{equation}
f({\bf x} + {\bf y}) - f({\bf x}) \geq \langle \nabla f({\bf x}), {\bf y} \rangle
\end{equation}
holds for all $k$-sparse vectors ${\bf y}$.
\label{definition_restricted_bregman_divergence}
\end{definition}

\begin{definition}[Stable Restricted Linearization (SRL)~\cite{bahmani2013greedy}] Let ${\bf x}$ be a $k$-sparse vector in $\mathbb{R}^p$. For function $f: \mathbb{R}^p \rightarrow \mathbb{R}$ we define the functions
\begin{equation}
\alpha_k ({\bf x}) = \textbf{sup} \Bigg\{ \frac{1}{\| {\bf y} \|_2^2} {B}_{f} ({\bf x} + {\bf y} \| {\bf x}) \Big| {\bf y} \neq {\bf 0} \text{ and } |supp({\bf x}) \cup supp({\bf y}) | \leq k \Bigg\}
\end{equation}
and
\begin{equation}
\beta_k ({\bf x}) = \textbf{inf} \Bigg\{ \frac{1}{\| {\bf y} \|_2^2} {B}_{f} ({\bf x} + {\bf y} | {\bf x}) \Big| {\bf y} \neq {\bf 0} \text{ and } |supp({\bf x}) \cup supp({\bf y}) | \leq k \Bigg\}
\end{equation}
 Then $f(\cdot)$ is said to have a Stable Restricted Linearization with constant $\mu_k$, or $\mu_k$-\textbf{SRL} if $\frac{\alpha_k({\bf x})}{\beta_{k} ({\bf x})} \leq \mu_k$
 \label{definition_restricted_linearization}
\end{definition}

\begin{lemma}
Denote $\Delta = {\bf x_1} - {\bf x_2}, \Delta' = \nabla f({\bf x_1}) - \nabla f({\bf x_2})$, and let $ r \geq | supp({\bf x_1}) \cup supp({\bf x_2})|$, $\bar{\alpha}_l({\bf x_1,x_2}) = \alpha_l ({\bf x_1}) + \alpha_l ({\bf x_2})$, $\bar{\beta}_l ({\bf x_1,x_2}) = \beta_l ({\bf x_1}) + \beta_l ({\bf x_2})$, $\bar{\gamma}_l ({\bf x_1,x_2}) = \bar{\alpha}_l({\bf x_1,x_2}) - \bar{\beta}_l({\bf x_1,x_2})$. For any $R' \subseteq R = supp({\bf x_1} - {\bf x_2})$, we have 
\begin{align}
\| {\Delta'}_{R'} \| &\leq \bar{\alpha}_r \| \Delta_{R'}\|_2 + \bar{\gamma}_r\| \Delta \|_2  \\
\| {\Delta'}_{R'} \| &\geq \bar{\beta}_r \| \Delta_{R'}\|_2 - \bar{\gamma}_r\| \Delta_{R\backslash R'} \|_2 \label{equation_11_0}
\end{align}
\label{lemma_4.2}
\end{lemma}
\vspace{-6mm}
\begin{proof}
We can get the following properties 
\begin{align}
\Big| \bar{\alpha}_r \| \Delta_{R'} \|_2^2 - \langle \Delta',\Delta_{R'} \rangle \Big| &\leq \bar{\gamma}_r \| \Delta_{R'} \|_2 \| \Delta \|_2 
\label{equation_11_1}
\\
\Big|  \| {\Delta'}_{R'}\|_2^2 - \bar{\alpha}_r \langle \Delta',\Delta_{R'} \rangle \Big| &\leq \bar{\gamma}_r \| {\Delta'}_{R'} \|_2 \| \Delta \|_2\label{equation_12_1}
\end{align}
from \cite{bahmani2013greedy}, where $R'$ be a subset of $R = \text{supp}(\Delta)$. It follows from (\ref{equation_11_1}) and (\ref{equation_12_1}) that 
\begin{align}
\| {\Delta'}_{R'}\|_2^2 - \bar{\alpha}_r^2\|  \Delta_{R'}\|_2^2 &= \|  {\Delta'}_{R'}\|_2^2 - \bar{\alpha}_r \langle \Delta',  \Delta_{R'} \rangle + \bar{\alpha}_r \Big[ - \bar{\alpha}_r \| \Delta_{R'}\|_2^2 + \langle \Delta',  \Delta_{R'} \rangle \Big] \nonumber \\
&\leq \bar{\gamma}_r \|  {\Delta'}_{R'}\|_2 \| \Delta \|_2 + \bar{\alpha}_r \bar{\gamma}_r \|  \Delta_{R'} \|_2 \| \Delta \|_2 \nonumber .
\end{align}
It can be reformulated as the following
\begin{align}
 \|  {\Delta'}_{R'}\|_2^2 - \bar{\gamma}_r \|  {\Delta'}_{R'}\|_2 \| \Delta \|_2  &\leq  \bar{\alpha}_r^2\|  \Delta_{R'}\|_2^2 + \bar{\alpha}_r \bar{\gamma}_r \|  \Delta_{R'} \|_2 \| \Delta \|_2 \nonumber \\
\|  {\Delta'}_{R'}\|_2^2 - \bar{\gamma}_r \|  {\Delta'}_{R'}\|_2 \| \Delta \|_2 + \frac{1}{4}{\bar{\gamma}_r}^2 \| \Delta \|_2^2 &\leq \bar{\alpha}_r^2\|  \Delta_{R'}\|_2^2  + \bar{\alpha}_r \bar{\gamma}_r \|  \Delta_{R'} \|_2 \| \Delta \|_2 + \frac{1}{4}{\bar{\gamma}_r}^2 \| \Delta \|_2^2  \nonumber \\
(\| \Delta'_{R'} \|_2 - \frac{1}{2} \bar{\gamma}_r \| \Delta \|_2 )^2 &\leq ( \bar{\alpha}_r \| \Delta_{R'} \|_2 + \frac{1}{2} \bar{\gamma}_r \| \Delta\|_2 )^2
\end{align}
Hence, we have $\| \Delta'_{R'}\|_2 \leq \bar{\alpha}_r \| \Delta_{R'} \|_2 + \bar{\gamma}_r \| \Delta \|_2$. We directly get~(\ref{equation_11_0}) from \cite{bahmani2013greedy}.
\end{proof}

\begin{theorem}
Suppose that $f$ satisfies $\mu_{8k}$-SRL with $\mu_{8k} \leq 1+\sqrt{\frac{1}{56}}$. Furthermore, suppose for $\beta_{8k}$ in Definition~\ref{definition_restricted_linearization} exists some $\epsilon > 0$ such that $\beta_{8k} \geq \epsilon$ holds for all $8k$-sparse vectors ${\bf x}$. Then ${\bf x}^{i+1}$, the estimate at the $i+1$-th iteration, satisfies. for any true ${\bf x} \in \mathbb{R}^n$ with $\text{supp}({\bf x}) \in \mathbb{M}(k,g)$, the iterates of Algorithm~\ref{alg:Graph-MP} must obey
\begin{equation}
\| {\bf r}^{i+1}\| \leq \sigma \|{\bf r} ^i\| + \nu  \| \nabla_I f({\bf x}) \|_2,
\end{equation}
where $\sigma = \sqrt{ \mu_{8k}^2 - \Big( 2+ c_H - 2 \mu_{8k}\Big)^2}$ and $\nu = \frac{ (2+c_H - 2 \mu_{8k})(1+c_H)+ \sigma}{2 \epsilon \sigma }$.
\label{theorem_4.2_SRL}
\end{theorem}

\begin{proof}
Let ${\bf r}^{i+1} = {\bf x}^{i+1} - {\bf x}$. $\| {\bf r}^{i+1} \|_2$ is upper bounded as 
\begin{align}
    \| {\bf r}^{i+1} \|_2 = \| {\bf x}^{i+1} - {\bf x} \|_2 &\leq \| {\bf x}^{i+1} - {\bf b} \|_2 + \| {\bf x} - {\bf b} \|_2 \nonumber \\
    &\leq c_T \| {\bf x} - {\bf b} \|_2 + \| {\bf x} - {\bf b}\|_2 \nonumber \\
    &= (1+c_T)\| {\bf x} - {\bf b} \|_2. \nonumber
\end{align}
The first inequality above follows by the triangle inequality and the second inequality follows by tail approximation. Since $\Omega = \Gamma \cup \text{supp}({\bf x}^i)$ and 
${\bf b} = \arg\min_{{\bf x} \in \mathbb{R}^n} f({\bf x})\ \ s.t. \ \ {\bf x}_{\Omega^c} = {\bf 0}$, we have 
\begin{align}
    \| {\bf x - b} \|_2 &\leq \| ({\bf x -b})_{\Omega^c} \|_2 + \|({\bf x -b})_{\Omega} \|_2 \nonumber \\
    &= \| {\bf x}_{\Omega^c} \|_2 + \|({\bf x -b})_{\Omega} \|_2 \nonumber \\
    &= \| ({\bf x - x}^{i})_{\Omega^c} \|_2 + \|({\bf x -b})_{\Omega} \|_2 \nonumber \\
    &= \| {{\bf r}^{i}_{\Omega^c} } \|_2 + \|({\bf x -b})_{\Omega} \|_2 \nonumber
\end{align}
Since ${\bf b}$ satisfies ${\bf b} = \arg\min_{{\bf x} \in \mathbb{R}^n} f({\bf x})\ \ s.t. \ \ {\bf x}_{\Omega^c} = 0$, we must have ${\nabla f({\bf b})}|_{\Omega} = {\bf 0}$. Then it follows from Corollary 2 in~\cite{bahmani2013greedy},
\begin{align}
    \| (\nabla f({\bf x}) - \nabla f({\bf b}))_{\Omega} \|_2 &\geq \bar{\beta}_{6k} \| ({\bf x -b} )_{\Omega} \|_2 - \bar{\gamma}_{6k} \| ({\bf x - b})_{\Omega^c} \|_2 \nonumber \\
    \| \nabla_{\Omega} f({\bf x}) \|_2 &\geq \bar{\beta}_{6k} \| ({\bf x -b} )_{\Omega} \|_2 - \bar{\gamma}_{6k} \| ({\bf x - b})_{\Omega^c} \|_2 \nonumber \\
    \| \nabla_{\Omega} f({\bf x}) \|_2 &\geq \bar{\beta}_{6k} \| ({\bf x -b} )_{\Omega} \|_2 - \bar{\gamma}_{6k} \| ({\bf x} - {\bf x}^i)_{\Omega^c} \|_2, \nonumber
\end{align}
where $\bar{\alpha}_{6k}({\bf x_1,x_2}) = \alpha_{6k} ({\bf x_1}) + \alpha_{6k} ({\bf x_2})$, $\bar{\beta}_{6k}({\bf x_1,x_2}) = \beta_{6k} ({\bf x_1}) + \beta_{6k} ({\bf x_2})$ and $\bar{\gamma}_{6k}({\bf x_1,x_2}) = \bar{\alpha}_{6k}({\bf x_1,x_2}) - \bar{\beta}_{6k}({\bf x_1,x_2})$. As $|supp({\bf x -b}) | \leq 6k$, we have $6k$-sparsity by Definition~(\ref{definition_restricted_linearization}). Note that $\Omega \cap R$ is a subset of $R$ and $\| (\nabla f({\bf x}) - \nabla f({\bf b}))_{\Omega} \|_2 \geq \| (\nabla f({\bf x}) - \nabla f({\bf b}))_{\Omega \cap R} \|_2$. Similarly, we have $({\bf x - b})_{\Omega} = ({\bf x - b})_{\Omega \cap R} $ and $({\bf x - b})_{\Omega^c} = ({\bf x - b})_{R \backslash (\Omega \cap R)} $. The second inequality follows by ${\nabla_{\Omega} f({\bf b})} = {\bf 0}$, and the third inequality follows by ${\bf b}_{\Omega^c} = {\bf 0}$ and ${\bf x}^{i}_{\Omega^c} = {\bf 0}$. Therefore, $\| {\bf x -b}\|_2$ can be further upper bounded as
\begin{align}
    \|{\bf x - b}\|_2 &\leq  \| {{\bf r}^{i}_{\Omega^c} } \|_2 + \|({\bf x -b})_{\Omega} \|_2 \nonumber \\
    &\leq \| {{\bf r}^{i}_{\Omega^c} } \|_2 + \frac{\bar{\gamma}_{6k} \| ({\bf x} - {\bf x}^i)_{\Omega^c} \|_2 }{\bar{\beta}_{6k}} + \frac{\| \nabla f({\bf x})_{\Omega}\|_2}{\bar{\beta}_{6k}} \nonumber \\
    &= \Big[  1 +\frac{\bar{\gamma}_{6k}}{\bar{\beta}_{6k}} \Big]\| {{\bf r}^{i}_{\Omega^c} } \|_2 + \frac{\| \nabla f({\bf x})_{\Omega}\|_2}{\bar{\beta}_{6k}}
\end{align}
Let $R = \text{supp}({\bf x}^i - {\bf x})$ and $\Gamma = \textbf{H}(\nabla f({\bf x}^i)) \in \mathbb{M}^{+} = \{H \cup T|H\in \mathbb{M}(k_H,g),T \in \mathbb{M}(k_T,g)\}$. We notice that $R\in \mathbb{M}^{+}$. The component $\| \nabla_{\Gamma} f({\bf x}^i)\|_2$ can be lower bounded as
\begin{align}
    \| \nabla_{\Gamma} f({\bf x}^i)\|_2 &\geq  c_H \| \nabla_{R} f({\bf x}^i) \|_2 \nonumber \\
    &\geq c_H\| \nabla_{R} f({\bf x}^i) -\nabla_{R} f({\bf x}) \|_2 - c_H \| \nabla_{R} f({\bf x}) \|_2 \nonumber \\
    &\geq c_H \bar{\beta}_{6k} \| {\bf x}^i - {\bf x} \|_2 - c_H \| \nabla_{I} f({\bf x}) \|_2 \nonumber \\
    &= c_H \bar{\beta}_{6k} \| {\bf r}^i \|_2 - c_H \| \nabla_{I} f({\bf x}) \|_2 
    %&\geq \| \nabla_{\Phi \cap R} f({\bf x}^i) -\nabla_{\Phi \cap R} f({\bf x}) \|_2 - \| \nabla_{I} f({\bf x})  \|_2 \nonumber \\
    %&\geq \bar{\beta}_{3k} \| ({\bf x^i - x})_{\Phi \cap R}\|_2 - \bar{\gamma}_{3k} \| ({\bf x^i - x})_{R \backslash (\Phi \cap R)} \|_2 - \| \nabla_{I} f({\bf x})  \|_2 \nonumber \\
    %&= \bar{\beta}_{3k} \| ({\bf x^i - x})_{\Phi}\|_2 - \bar{\gamma}_{3k} \| ({\bf x^i - x})_{\Phi^c} \|_2 - \| \nabla_{I} f({\bf x})  \|_2  \nonumber \\
    %&= \bar{\beta}_{3k} \| {\bf r}^{i}_{\Phi}\|_2 - \bar{\gamma}_{3k} \| {\bf r}^{i}_{\Phi^c} \|_2 - \| \nabla_{I} f({\bf x})  \|_2 \nonumber \\
    %&\geq \textcolor{red}{\bar{\beta}_{3k} \| {\bf r}^{i}\|_2 - \bar{\gamma}_{3k} \| {\bf r}^{i} \|_2 - \| \nabla_{I} f({\bf x})  \|_2} 
    \label{equ_15}
\end{align}
The first inequality follows the head approximation and $R \in \mathbb{M}^{+}$. The second one is from triangle inequality and the third one follows by Lemma~(\ref{lemma_4.2}). The component $\| \nabla_{\Gamma} f ({\bf x}^i) \|_2$ can also be upper bounded as
\begin{align}
    \| \nabla_{\Gamma} f ({\bf x}^i) \|_2 &\leq \| \nabla_{\Gamma} f ({\bf x}^i) - \nabla_{\Gamma} f({\bf x}) \|_2 + \| \nabla_{\Gamma} f({\bf x}) \|_2 \nonumber \\
    &\leq \| \nabla_{\Gamma \backslash R^c} f ({\bf x}^i) - \nabla_{\Gamma \backslash R^c} f({\bf x}) +  \nabla_{\Gamma \cap R^c} f ({\bf x}^i) - \nabla_{\Gamma \cap R^c} f({\bf x})  \|_2 + \| \nabla_{\Gamma} f({\bf x}) \|_2 \nonumber \\
    &\leq \| \nabla_{\Gamma \backslash R^c} f ({\bf x}^i) - \nabla_{\Gamma \backslash R^c} f({\bf x}) \|_2 + \| \nabla_{\Gamma \cap R^c} f ({\bf x}^i) - \nabla_{\Gamma \cap R^c} f({\bf x})  \|_2 + \| \nabla_{\Gamma} f({\bf x}) \|_2 \nonumber \\
    &\leq \| \nabla_{\Gamma \backslash R^c} f ({\bf x}^i) - \nabla_{\Gamma \backslash R^c} f({\bf x}) \|_2 + \bar{\gamma}_{8k} \| {\bf r}^i \|_2 + \| \nabla_{\Gamma} f({\bf x}) \|_2 \nonumber \\
    &\leq \bar{\alpha}_{6k} \| {\bf r}_{\Gamma \backslash R^c}^{i}\|_2 + \bar{\gamma}_{6k} \| {\bf r}^i \|_2 + \bar{\gamma}_{8k} \| {\bf r}^i \|_2 + \| \nabla_{I} f({\bf x}) \|_2 
    \label{equ_16}
\end{align}
The first and third inequalities follow by the triangle inequality. The second inequality follows by $\Gamma = (\Gamma \cap R^c) \cup (\Gamma \backslash R^c)$. And the last inequality follows by $\|( f({\bf x}^i) - f({\bf x}))_{R'}\|_2 \leq \bar{\gamma}_{k+r} \| {\bf x}^i - {\bf x}\|_2$, where $k \leq |R'|, r = |supp({\bf x}^i - {\bf x}) |$ and $R' \subseteq R^c$.
By Lemma~(\ref{lemma_4.2}), we have $\| \nabla_{\Gamma \backslash R^c} f ({\bf x}^i) - \nabla_{\Gamma \backslash R^c} f({\bf x}) \|_2 \leq \bar{\alpha}_{6k} \| {\bf r}_{\Gamma \backslash R^c}^{i}\|_2 + \bar{\gamma}_{6k} \| {\bf r}^i \|_2$. Combining Equation~(\ref{equ_15}) and Equation~(\ref{equ_16}), we have
\begin{align}
    c_H\bar{\beta}_{6k} \| {\bf r}^{i}\|_2 - c_H \| \nabla_{I} f({\bf x})  \|_2  &\leq  \bar{\alpha}_{6k} \| {\bf r}_{\Gamma \backslash R^c}^{i}\|_2 + \bar{\gamma}_{6k} \| {\bf r}^i \|_2 + \bar{\gamma}_{8k} \| {\bf r}^i \|_2 + \| \nabla_{I} f({\bf x}) \|_2 \nonumber \\
    c_H \bar{\beta}_{3k} \| {\bf r}^{i}\|_2  - c_H \| \nabla_{I} f({\bf x})  \|_2  &\leq  \bar{\alpha}_{6k} \| {\bf r}_{\Gamma}^i\|_2 + \bar{\gamma}_{6k} \| {\bf r}^i \|_2 + \bar{\gamma}_{8k} \| {\bf r}^i \|_2 + \| \nabla_{I} f({\bf x}) \|_2 \nonumber \\
    (c_H \bar{\beta}_{3k} - \bar{\gamma}_{6k} - \bar{\gamma}_{8k}) \| {\bf r}^{i} \|_2 - (1+c_H) \| \nabla_{I} f({\bf x})  \|_2  &\leq  \bar{\alpha}_{6k} \| {\bf r}_{\Gamma}^{i}\|_2 \nonumber \\
     \mu_{8k} \| {\bf r}_{\Gamma}^{i}\|_2  &\geq (c_H+2 - 2\mu_{8k}) \| {\bf r}^{i} \|_2 - \frac{1+c_H}{2\epsilon} \| \nabla_{I} f({\bf x})  \|_2  \nonumber
\end{align}
Finally, we get $\| {\bf r}_{\Gamma}^i \| \geq \Big(\frac{2+c_H}{\mu_{8k}} - 2 \Big) \|{\bf r}^i \| - \frac{1+c_H}{2 \epsilon \mu_{8k}} \| \nabla_I f({\bf x})\|$. Let us assume the SRL parameter $\mu_{8k} \leq \frac{2 + c_H}{2}$. Using the same computing procedure of Lemma 9 in~\cite{hegde2015approximation}, we have
\begin{equation}
\| {\bf r}_{\Gamma^c}^i \|_2 \leq \eta \| {\bf r}^i\| +  \frac{(2+c_H - 2 \mu_{8k} )(1+c_H)}{2 \epsilon \mu_{8k}^2 \eta } \| \nabla_{I} f({\bf x})\|_2,
\end{equation}
where $\eta =  \sqrt{ 1 - (\frac{2+c_H}{\mu_{8k}} - 2)^2}$. Combine them together, we have
\begin{align}
\|{\bf x - b}\|_2 &\leq \Big(  1 +\frac{\bar{\gamma}_{6k}}{\bar{\beta}_{6k}} \Big)\| {{\bf r}^{i}_{\Omega^c} } \|_2 + \frac{\| \nabla f({\bf x})_{\Omega}\|_2}{\bar{\beta}_{6k}} \nonumber \\
 &\leq \mu_{8k} \| {{\bf r}^{i}_{\Omega^c} } \|_2 + \frac{\| \nabla f({\bf x})_{\Omega}\|_2}{\bar{\beta}_{6k}} \nonumber \\
  &\leq \mu_{8k} \| {{\bf r}^{i}_{\Gamma^c} } \|_2 + \frac{\| \nabla f({\bf x})_{\Omega}\|_2}{2\epsilon} \nonumber \\
&\leq \sigma \|{\bf r} ^i\| + \nu \| \nabla_I f({\bf x}) \|_2 ,
\end{align}
where $\sigma = \sqrt{ \mu_{8k}^2 - \Big( 2+ c_H - 2 \mu_{8k}\Big)^2}$ and $\nu = \frac{ (2+c_H - 2 \mu_{8k})(1+c_H)+ \sigma}{2 \epsilon \sigma }$. Hence, we prove this theorem.
\end{proof}

\begin{theorem}
Let the true parameter be ${\bf x} \in \mathbb{R}^n$ such that $\text{supp}({\bf x}) \in \mathbb{M}(k, g)$, and $f: \mathbb{R}^n \rightarrow \mathbb{R}$ be cost function that satisfies SRL condition. The \textsc{Graph-MP} algorithm returns a $\hat{\bf x}$ such that,  $\text{supp}(\hat{\bf x}) \in \mathbb{M}(5k, g)$ and $\|{\bf x} - \hat{\bf x}\|_2 \le c \|\nabla_I f({\bf x})\|_2$, where $c= (1+\frac{\nu}{1-\sigma})$ and $I = \arg \max_{S \in \mathbb{M}(8k, g)} \|\nabla_S f({\bf x})\|_2$. The parameters $\sigma$ and $\nu$ are fixed constants defined in Theorem~\ref{theorem_4.2_SRL}. Moreover, \textsc{Graph-MP} runs in time 
\begin{eqnarray}
O\left((T+|\mathbb{E}|\log^3 n) \log (\|{\bf x}\|_2/ \|\nabla_I f({\bf x})\|_2)\right) \label{timecomplexity-1},
\end{eqnarray}
where $T$ is the time complexity of one execution of the subproblem in Step 6 in \textsc{Graph-MP}. In particular, if $T$ scales linearly with $n$, then \textsc{Graph-MP} scales nearly linearly with $n$. 
\end{theorem}

\begin{proof}
The i-th iterate of \textsc{Graph-MP} satisfies
\vspace{-1mm}
\begin{small}
\begin{eqnarray}
\|{\bf x} - {\bf x}^i\|_2 \le \sigma^i \|{\bf x}\|_2 + \frac{\nu}{1-\sigma} \|\nabla_I f({\bf x})\|_2.
\end{eqnarray}
\end{small}

\vspace{-3mm}
\noindent After $t = \left \lceil \log \left(\frac{\|{\bf x}\|_2}{\|\nabla_I f({\bf x})\|_2}\right) / \log \frac{1}{\sigma} \right \rceil $ iterations,
\textsc{Graph-MP} returns an estimate $\hat{x}$ satisfying $\|{\bf x} - \hat{{\bf x}}\|_2 \le (1 + \frac{\nu}{1 - \sigma}) \|\nabla_I f({\bf x})\|_2$ as $\sigma < 1$ and the summation of $\sum_{k = 0}^{i} \nu \sigma^k = \frac{\nu(1-\sigma^i)}{1-\sigma} \leq \frac{\nu}{1-\sigma}$. The time complexities of both head approximation and tail approximation are $O(|\mathbb{E}| \log^3 n)$. The time complexity of one iteration in \textsc{Graph-MP} is $(T+|\mathbb{E}|\log^3 n)$, and the total number of iterations is $\left \lceil \log \left(\frac{\|{\bf x}\|_2}{\|\nabla_I f({\bf x})\|_2}\right) / \log \frac{1}{\alpha} \right \rceil $, and hence the overall time follows.
\end{proof}

\section{Theoretical Analysis of \textsc{Graph}-\textsc{Mp} under RSC/RSS condition}
\label{theoremtical_1}

\begin{definition}[Restricted Strong Convexity/Smoothness, ($m_k,M_k,\mathbb{M}$)-RSC/RSS]\cite{yuan2013gradient}. For any integer $k>0$, we say $f({\bf x})$ is restricted $m_k$-strongly convex and $M_k$-strongly smooth of there exist $\exists m_k$, $M_k >0$ such that 
\begin{equation}
\frac{m_k}{2} \| {\bf x} -{\bf y} \|_2^2 \leq f({\bf x}) - f({\bf y}) - \langle \nabla f({\bf y}),{\bf x}-{\bf y} \rangle \leq \frac{M_k}{2} \| {\bf x} -{\bf y} \|_2^2, \forall \| {\bf x} - {\bf y} \|_0 \leq k
\end{equation}
\label{definition_RSC}
\end{definition}

\begin{lemma} Let $S$ be any index set with cardinality $|S| \leq k$ and $S \in \mathbb{M}(k,g)$. If $f$ is \textbf{($m_k,M_k,\mathbb{M}$)-RSC/RSS}, then $f$ satisfies the following property
\begin{equation}
\|{\bf x} - {\bf y} - \frac{m_k}{M_k^2} \Big( \nabla_S f({\bf x}) - \nabla_S f({\bf y}) \Big) \|_2 \le \sqrt{1- (\frac{m_k}{M_k})^2} \|{\bf x} - {\bf y}\|_2
\end{equation}
\label{lemma_3}
\end{lemma}
\begin{proof}
By adding two copies of the inequality~(\ref{definition_RSC}) with ${\bf x}$ and ${\bf y}$, we have
\begin{equation}
m_k \| {\bf x} -{\bf y} \|_2^2 \leq \langle \nabla f({\bf x}) - \nabla f({\bf y}),{\bf x}-{\bf y}\rangle \leq M_k \| {\bf x} -{\bf y} \|_2^2, \forall \| {\bf x} - {\bf y} \|_0 \leq k .
\label{equation_3}
\end{equation}
\noindent By Theorem 2.1.5 in~\cite{nesterov2013introductory}, we have $\langle \nabla f({\bf x}) - \nabla f({\bf y}),{\bf x}-{\bf y}\rangle \geq \frac{1}{L} \| \nabla f({\bf x}) - \nabla f({\bf y})\|_2^2$, which means
\begin{equation}
\| \nabla_S f({\bf x}) - \nabla_S f({\bf y}) \|_2^2 \leq  \| \nabla f({\bf x}) - \nabla f({\bf y}) \|_2^2 \leq M_k L \| {\bf x} - {\bf y} \|_2^2.
\end{equation}
Let $L = M_k$ and then $\| \nabla_S f({\bf x}) - \nabla_S f({\bf y}) \|_2 \leq M_k \| {\bf x} - {\bf y} \|_2$. The left side of inequality~(\ref{equation_3}) is
\begin{equation}
m_k \| {\bf x} -{\bf y} \|_2^2 \leq \langle \nabla f({\bf x}) - \nabla f({\bf y}),{\bf x}-{\bf y}\rangle = ({\bf x} - {\bf y})^T (\nabla_S f({\bf x}) - \nabla_S f({\bf y})). 
\label{equation_11}
\end{equation}
\noindent The last equation of (~\ref{equation_11}) follows by ${\bf x-y} = ({\bf x -y})_{S}$. For any ${\bf a}$ and ${\bf b}$, we have $\| {\bf a} - {\bf b} \|_2^2 = \| {\bf a} \|_2^2 + \| {\bf b} \|_2^2 - 2 {\bf a}^T {\bf b}$. By replacing ${\bf a}$ as $({\bf x} - {\bf y})$ and ${\bf b}$ as $\frac{m_k}{M_k^2} \Big( \nabla_S f({\bf x}) - \nabla_S f({\bf y}) \Big)$, we have

\begin{eqnarray}
\| {\bf x} - {\bf y} -\frac{m_k}{M_k^2} \Big( \nabla_S f({\bf x}) - \nabla_S f({\bf y}) \Big) \|_2^2 & =& \| {\bf x} - {\bf y} \|_2^2 + \frac{m_k^2}{M_k^4} \| \nabla_S f({\bf x}) - \nabla_S f({\bf y}) \|_2^2 \\
& -& \frac{2m_k}{M_k^2}({\bf x} - {\bf y})^T(\nabla_S f({\bf x}) - \nabla_S f({\bf y})) \nonumber \\
& \leq& (1+ \frac{m_k^2}{M_k^2}- \frac{2m_k^2}{M_k^2}) \| {\bf x} - {\bf y} \|_2^2 \nonumber \\
& =& (1 - \frac{m_k^2}{M_k^2}) \| {\bf x} - {\bf y} \|_2^2.
\label{inequal_}
\end{eqnarray}
By taking the square root for both sides of~(\ref{inequal_}), we can prove the result. If one follows Lemma 1 in \cite{yuan2013gradient} by replacing $\delta$ as $\frac{m_k}{M_k^2}$ and $\rho_s $ as $\sqrt{1- (\frac{m_k}{M_k})^2}$, one can also get same result.
\end{proof}

\begin{theorem}Consider the graph-structured sparsity model $\mathbb{M}(k, g)$ for some $k, g\in \mathbb{N}$ and a cost function $f: \mathbb{R}^n \rightarrow \mathbb{R}$ that satisfies condition $\left(m_k, M_k, \mathbb{M}(8k, g)\right)$-RSC/RSS. If \begin{small}$\alpha_0 = c_H - \sqrt{1 - \frac{m_k^2}{M_k^2}} \cdot (1 + c_H), $\end{small} then for any true ${\bf x} \in \mathbb{R}^n$ with $\text{supp}({\bf x}) \in \mathbb{M}(k, g)$, the iterates of Algorithm 1 obey
\begin{small}
\begin{eqnarray}
\|{\bf x}^{i+1}-{\bf x}\|_2 \le \frac{M_k(1+c_T)\sqrt{1 - \alpha_0^2}}{M_k -\sqrt{M_k^2 - m_k^2}}  \cdot \|{\bf x}^i-{\bf x}\|_2 + \frac{m_k(1 + c_T)}{M_k^2 - M_k \sqrt{M_k^2 - m_k^2}} \Big(\frac{1 + c_H+\alpha_0}{\alpha_0} + \frac{\alpha_0 (1 + c_H)}{\sqrt{1 - \alpha_0^2}} \Big)\|\nabla_I f({\bf x})\|_2,\nonumber \label{decay-rate-0}
\end{eqnarray}
\end{small}
\noindent where \begin{small}$ I = \arg \max_{S \in \mathbb{M}(8k, g)} \|\nabla_{S} f({\bf x})\|_2 $\end{small}
\label{theorem:convergence-0}
\end{theorem}
\begin{proof}
Let ${\bf r}^{i+1} = {\bf x}^{i+1} - {\bf x}$. $\|{\bf r}^{i+1}\|_2$ is upper bounded as 
\begin{small}
\begin{eqnarray}
\|{\bf r}^{i+1}\| = \|{\bf x}^{i+1} - {\bf x}\|_2 
&\le & \|{\bf x}^{i+1} - {\bf b}\|_2  + \|{\bf x} - {\bf b}\|_2 \nonumber\\
&\le& c_T \|{\bf x} - {\bf b}\|_2  + \|{\bf x} - {\bf b}\|_2 \nonumber\\
&\le& (1 + c_T) \|{\bf x} - {\bf b}\|_2, \nonumber
\end{eqnarray}
\end{small}

\vspace{-4mm}
\noindent which follows from the definition of tail approximation. 
The component $\|({\bf x} - {\bf b})_\Omega\|_2^2$ is upper bounded as 
\begin{small}
\begin{eqnarray}
\|({\bf x} - {\bf b})_\Omega\|_2^2 &= & \langle {\bf b} - {\bf x}, ({\bf b} - {\bf x})_\Omega \rangle \nonumber \\ 
&= & \langle {\bf b} - {\bf x} - \frac{m_k}{M_k^2} \nabla_\Omega f({\bf b}) + \frac{m_k}{M_k^2} \nabla_\Omega f({\bf x}), ({\bf b} - {\bf x})_\Omega\rangle -     \langle \frac{m_k}{M_k^2} \nabla_\Omega f({\bf x}), ({\bf b} - {\bf x})_\Omega \rangle \nonumber \\
 &\le& \sqrt{1-\frac{m_k^2}{M_k^2}} \|{\bf b} - {\bf x}\|_2\cdot \|({\bf b} - {\bf x})_\Omega\|_2 + \frac{m_k}{M_k^2} \|\nabla_\Omega f({\bf x})\|_2 \cdot \|({\bf b} - {\bf x})_\Omega\|_2 \nonumber,
\end{eqnarray}
\end{small}

\vspace{-3mm}
\noindent where the second equality follows from the fact that $\nabla_\Omega f({\bf b}) = 0$ since ${\bf b}$ is the solution to the problem in Step 6 of Algorithm 1, and the last inequality follows from Lemma~\ref{lemma_3}. 
After simplification, we have 
\[\|({\bf x} - {\bf b})_\Omega\|_2 \le \sqrt{1-\frac{m_k^2}{M_k^2}} \|{\bf b} - {\bf x}\|_2 + \frac{m_k}{M_k^2} \|\nabla_\Omega f({\bf x})\|_2\]
\noindent It follows that
\begin{small}\begin{eqnarray}
\|{\bf x} - {\bf b}\|_2 \le \|({\bf x} - {\bf b})_\Omega\|_2 + \|({\bf x} - {\bf b})_{\Omega^c}\|_2 
\le \sqrt{1-\frac{m_k^2}{M_k^2}} \|{\bf b} - {\bf x}\|_2 + \frac{m_k}{M_k^2} \|\nabla_\Omega f({\bf x})\|_2 + \|({\bf x} - {\bf b})_{\Omega^c}\|_2 \nonumber 
\end{eqnarray}\end{small}
\noindent After rearrangement we obtain 
\begin{small}
\begin{eqnarray}
\|{\bf b} - {\bf x}\|_2 & \le& \frac{M_k}{M_k - \sqrt{M_k^2 - m_k^2}} \Big( \|({\bf b} - {\bf x})_{\Omega^c}\|_2 + \frac{m_k}{M_k^2} \|\nabla_\Omega f({\bf x})\|_2 \Big) \nonumber \\
& =& \frac{M_k}{M_k - \sqrt{M_k^2 - m_k^2}} \Big( \|{\bf x}_{\Omega^c}\|_2 + \frac{m_k}{M_k^2} \|\nabla_\Omega f({\bf x})\|_2 \Big) \nonumber \\
& =& \frac{M_k}{M_k - \sqrt{M_k^2 - m_k^2}} \Big( \|({\bf x} - {\bf x}^i)_{\Omega^c}\|_2 + \frac{m_k}{M_k^2} \|\nabla_\Omega f({\bf x})\|_2 \Big) \nonumber \\
& =& \frac{M_k}{M_k - \sqrt{M_k^2 - m_k^2}} \Big( \|{\bf r}_{\Omega^c}\|_2 + \frac{m_k}{M_k^2} \|\nabla_\Omega f({\bf x})\|_2 \Big) \nonumber \\
& \le& \frac{M_k}{M_k - \sqrt{M_k^2 - m_k^2}} \Big(\|{\bf r}^i_{\Gamma^c}\|_2 + \frac{m_k}{M_k^2} \|\nabla_\Omega f({\bf x})\|_2 \Big) \nonumber 
\end{eqnarray}
\end{small}
\noindent where the first equality follows from the fact that $\text{supp}({\bf b}) \subseteq \Omega$, the second and last inequalities follow from the fact that $\Omega = \Gamma \cup \text{supp}({\bf x}^i)$. Combining above inequalities, we obtain 
\begin{small}
\begin{eqnarray}
\|{\bf r}^{i+1}\|_2 \le \frac{M_k (1+c_T)}{M_k - \sqrt{M_k^2 - m_k^2}}  \Big(\|{\bf r}^i_{\Gamma^c}\|_2 + \frac{m_k}{M_k^2} \|\nabla_I f({\bf x})\|_2 \Big) \nonumber
\end{eqnarray}
\end{small}

\vspace{-3mm}
\noindent From Lemma~\ref{lemma:r-Complement-SRC}, we have 
\begin{small}
\begin{eqnarray}
\|{\bf r}^i_{\Gamma^c}\|_2  \le \sqrt{1 - \alpha_0^2} \|{\bf r}^i\|_2 +\left[\frac{\beta_0}{\alpha_0} + \frac{\alpha_0\beta_0}{\sqrt{1-\alpha_0^2}}\right] \|\nabla_I f({\bf x})\|_2\end{eqnarray}
\end{small}
\noindent Combining the above inequalities, we prove the theorem. 
\end{proof}

\begin{lemma}
Let ${\bf r}^i = {\bf x}^i - {\bf x}$ and $\Gamma = H(\nabla f({\bf x}^i))$. Then
\begin{small}
\begin{eqnarray}
\|{\bf r}^i_{\Gamma^c}\|_2  \le \sqrt{1 - \alpha_0^2} \|{\bf r}^i\|_2 +\left[\frac{\beta_0}{\alpha_0} + \frac{\alpha_0\beta_0}{\sqrt{1-\alpha_0^2}}\right] \|\nabla_I f({\bf x})\|_2\end{eqnarray}
\end{small}
\noindent ,where \begin{small}$\alpha_0 = c_H - \sqrt{1 - \frac{m_k^2}{M_k^2}} \cdot (1 + c_H), $\end{small} and \begin{small}$\beta_0 = \frac{m_k(1+c_H)}{M_k^2}$\end{small}, and \begin{small} $I = \arg \max_{S \in \mathbb{M}(8k, g)} \|\nabla_S f({\bf x})\|_2.$\end{small} We assume that $c_H$ and $\sqrt{1- \frac{m_s^2}{M_s^2}}$ are such that $\alpha_0 > 0$.  \label{lemma:r-Complement-SRC}
\end{lemma}
\begin{proof}
Denote $\Phi = \text{supp}({\bf x}) \in \mathbb{M}(k, g), \Gamma = H(\nabla f({\bf x}^i)) \in \mathbb{M}(2k, g)$, ${\bf r}^i = {\bf x}^i - {\bf x}$, and $\Omega = \text{supp}({\bf r}^i) \in \mathbb{M}(6k, g)$. The component $\|\nabla_\Gamma f({\bf x}^i)\|_2$ can be lower bounded as
\begin{small}
\begin{eqnarray}
\|\nabla_\Gamma f({\bf x}^i)\|_2 &\ge& c_H (\| \nabla_\Phi f({\bf x}^i)- \nabla_\Phi f({\bf x}) \|_2  - \|\nabla_\Phi f({\bf x})\|_2) \nonumber\\
&\ge& c_H   \frac{M_2^2 - M_k \sqrt{M_k^2 - m_k^2}}{m_k}  \|{\bf r}^i\|_2 - c_H \|\nabla_I f({\bf x})\|_2, \nonumber
\end{eqnarray} \end{small}

\vspace{-3mm}
\noindent where the last inequality follows from Lemma~\ref{lemma:twoinequalities}. 
The component $\|\nabla_\Gamma f({\bf x}^i)\|_2$ can also be upper bounded as
\vspace{-1mm}
\begin{small}
\begin{eqnarray}
\|\nabla_\Gamma f({\bf x}^i)\|_2 &\le& \frac{M_k^2}{m_k} \|\frac{m_k}{M_k^2} \nabla_\Gamma f({\bf x}^i)- \frac{m_k}{M_k^2}\nabla_\Gamma f({\bf x})\|_2 + \|\nabla_\Gamma f({\bf x})\|_2 \nonumber \\
&\le&  \frac{M_k^2}{m_k} \| \frac{m_k}{M_k^2}  \nabla_\Gamma f({\bf x}^i) - \frac{m_k}{M_k^2}  \nabla_\Gamma f({\bf x}) - {\bf r}^i_\Gamma + {\bf r}^i_\Gamma\|_2 + \|\nabla_\Gamma f({\bf x})\|_2 \nonumber\\
&\le &  \frac{M_k^2}{m_k}  \| \frac{m_k}{M_k^2} \nabla_{\Gamma\cup \Omega} f({\bf x}^i) - \frac{m_k}{M_k^2} \nabla_{\Gamma\cup \Omega} f({\bf x}) - {\bf r}^i_{\Gamma\cup \Omega}\|_2 + \frac{M_k^2}{m_k} \|{\bf r}^i_\Gamma\|_2 + \|\nabla_\Gamma f({\bf x})\|_2 \nonumber\\
&\le& \frac{M_k \sqrt{M_k^2 - m_k^2}}{m_k} \cdot  \|{\bf r}^i\|_2 + \frac{M_k^2}{m_k} \|{\bf r}^i_\Gamma\|_2 + \|\nabla_{I} f({\bf x})\|_2, \nonumber
\end{eqnarray}
\end{small}

\vspace{-3mm}
\noindent where the last inequality follows from condition $(\xi, \delta, \mathbb{M}(8k, g))$-RSC/RSS and the fact that ${\bf r}^i_{\Gamma\cup \Omega} = {\bf r}^i$. Combining the two bounds and grouping terms, we have 
\begin{eqnarray}
\|{\bf r}^i_\Gamma\|_2  &\ge& \alpha_0 \cdot  \|{\bf r}^i\|_2 - \beta_0 \cdot \|\nabla_I f(x)\|_2
\end{eqnarray}
,where $\alpha_0 = \Big[ c_H - \sqrt{1 - \frac{m_k^2}{M_k^2}} \cdot (1 + c_H) \Big]$ and $\beta_0 = \frac{m_k(1+c_H)}{M_k^2}$. 
We assume that the constant $\delta = \sqrt{1- \frac{m_k^2}{M_k^2}}$ is small enough such that $c_H > \frac{\delta}{1-\delta}$. We consider two cases.

\textbf{Case 1}: The value of $\|{\bf r}^i\|_2$ satisfies  $\alpha_0 \|{\bf r}^i\|_2 \le  \beta_0 \|\nabla f({\bf x})\|_2$. Then consider the vector ${\bf r}^i_{\Gamma^c}$. We have 
\begin{eqnarray}
\|{\bf r}^i_{\Gamma^c}\|_2 \le \frac{\beta_0}{\alpha_0} \|{\bf r}^i\|_2 \nonumber
\end{eqnarray}

\textbf{Case 2}: The value of $\|{\bf r}^i\|_2$  satisfies $\alpha_0  \|{\bf r}^i\|_2 \ge  \beta_0 \|\nabla f({\bf x})\|_2$. We get 
\[
\|{\bf r}^i_\Gamma\|_2 \ge \|{\bf r}^i\|_2 \left(\alpha_0 - \frac{\beta_0 \|\nabla_I f({\bf x})\|_2 }{\|{\bf r}^i\|_2} \right)
\]

Moreover, we also have $\|{\bf r}^i\|_2 = \|{\bf r}^i_\Gamma\|_2^2 + \|{\bf r}^i_{\Gamma^c}\|_2$. Therefore, we obtain 
\[
\|{\bf r}^i_{\Gamma^c}\|_2 \le \|{\bf r}^i\|_2 \sqrt{1 - \left(\alpha_0 - \frac{\beta_0 \|\nabla_I f({\bf x})\|_2 }{\|{\bf r}^i\|_2} \right)^2}. 
\]

We have the following inequality, for a given $0 < \omega_0 < 1$  and a free parameter $0 < \omega < 1$, a straightfoward calculation yields that $ \sqrt{1-\omega^2} \le \frac{1}{\sqrt{1 - \omega^2}} - \frac{\omega}{\sqrt{1-\omega^2}} \omega_0$. Therefore, substituting into the bound for $\|{\bf r}^i_{\Gamma^c}\|_2$, we get 
\begin{eqnarray}
\|{\bf r}^i_{\Gamma^c}\|_2  &\le& \|{\bf r}^i\|_2 \left(\frac{1}{\sqrt{1 - \omega^2}} - \frac{\omega}{\sqrt{1-\omega^2}} \left(\alpha_0 - \frac{\beta_0 \|\nabla_I f({\bf x})\|_2}{\|{\bf r}^i\|_2} \right)\right) \\
&=& \frac{1 - w\alpha_0 }{\sqrt{1 - \omega^2}} \|{\bf r}^i\|_2 + \frac{\omega\beta_0}{\sqrt{1-\omega^2}} \|\nabla_I f({\bf x})\|_2
\end{eqnarray}

The coefficient prceding $\|{\bf r}^i\|_2$ determines the overall convergence rate, and the minimum value of the coefficient is attained by setting $\omega = \alpha_0$. Substituting, we obtain 
\begin{equation}
\|{\bf r}^i_{\Gamma^c}\|_2  \le \sqrt{1 - \alpha_0^2} \|{\bf r}^i\|_2 +\left[\frac{\beta_0}{\alpha_0} + \frac{\alpha_0\beta_0}{\sqrt{1-\alpha_0^2}}\right] \|\nabla_I f({\bf x})\|_2,
\end{equation}
which proves the lemma. 

\end{proof}

\section{Theoretical Analysis of \textsc{Graph}-\textsc{Mp} under WRSC condition}
In order to demonstrate the accuracy of estimates using Algorithm 1 we require a variant of the \textit{Restricted Strong Convexity/Smoothness} (RSC/RSS) conditions proposed in~\cite{yuan2014icml} to hold. The RSC condition basically characterizes cost functions that have quadratic bounds on the derivative of the objective function when restricted to model-sparse vectors. The condition we rely on, the Weak Restricted Strong Convexity (WRSC), can be formally defined as follows:
\begin{definition}[Weak Restricted Strong Convexity  Property (WRSC)]
A function $f({\bf x})$ has condition ($\xi$, $\delta$, $\mathbb{M}$)-WRSC if $\forall {\bf x}, {\bf y} \in \mathbb{R}^n$  and $\forall  S \in \mathbb{M}$ with $\text{supp}({\bf x}) \cup \text{supp}({\bf y}) \subseteq S $, the following inequality holds for some $\xi > 0$ and $0 < \delta < 1$: 
\begin{eqnarray}
\|{\bf x} - {\bf y} - \xi \nabla_S f({\bf x}) + \xi \nabla_S f({\bf y})\|_2 \le \delta \|{\bf x} - {\bf y}\|_2. 
\end{eqnarray}
\end{definition}

\begin{remark}
1) In the special case where $f({\bf x}) = \|{\bf y} - A{\bf x}\|_2^2$ and $\xi = 1$, condition ($\xi$, $\delta$, $\mathbb{M}$)-WRSC reduces to the well known Restricted Isometry Property (RIP) condition in compressive sensing. 2) The RSC and RSS conditions imply condition WRSC, which indicates that condition WRSC is no stronger than the RSC and RSS conditions~\cite{yuan2014icml}. 
\end{remark}

\begin{lemma}\cite{yuan2014icml}
Assume that $f$ is a differentiable function. If $f$ satisfies condition $(\xi, \delta, \mathbb{M})$-WRSC, then $\forall {\bf x}, {\bf y} \in \mathbb{R}^n$ with $\text{supp}({\bf x})\cup \text{supp}({\bf y})\subset S \in \mathbb{M}$, the following two inequalities hold
\begin{small}
\begin{eqnarray}
\frac{1 - \delta}{\xi} \|{\bf x} - {\bf y}\|_2 \le \|\nabla_S f({\bf x}) - \nabla_S f({\bf y})\|_2 \le \frac{1 + \delta}{\xi} \|{\bf x} - {\bf y}\|_2, \nonumber \\
f({\bf x}) \le f({\bf y}) + \langle \nabla f({\bf y}), {\bf x} - {\bf y} \rangle + \frac{1+\delta}{2\xi} \|{\bf x} - {\bf y}\|_2^2. \nonumber 
 \end{eqnarray}
 \end{small}\label{lemma:twoinequalities}
 \vspace{-2mm}
\end{lemma}
\begin{lemma}
Let ${\bf r}^i = {\bf x}^i - {\bf x}$ and $\Gamma = H(\nabla f({\bf x}^i))$. Then
\begin{small}
\begin{eqnarray}
\|{\bf r}^i_{\Gamma^c}\|_2 \le \sqrt{1 - \eta^2} \|{\bf r}^i\|_2 + \left[\frac{\xi(1 + c_H)}{\eta} + \frac{\xi \eta (1 + c_H)}{\sqrt{1 - \eta^2}}\right] \|\nabla_I f({\bf x})\|_2,\nonumber
\end{eqnarray}
\end{small}

\vspace{-2mm}
\noindent where \begin{small}$\eta = c_H(1 - \delta) - \delta$\end{small} and \begin{small}
$I = \arg \max_{S \in \mathbb{M}(8k, g)} \|\nabla_S f({\bf x})\|_2.$\end{small} We assume that $c_H$ and $\delta$ are such that $\eta > 0$.  \label{lemma:r-Complement}
\end{lemma}
\begin{proof}
Denote $\Phi = \text{supp}({\bf x}) \in \mathbb{M}(k, g), \Gamma = H(\nabla f({\bf x}^i)) \in \mathbb{M}(2k, g)$, ${\bf r}^i = {\bf x}^i - {\bf x}$, and $\Omega = \text{supp}({\bf r}^i) \in \mathbb{M}(6k, g)$. The component $\|\nabla_\Gamma f({\bf x}^i)\|_2$ can be lower bounded as
\begin{small}
\begin{eqnarray}
\|\nabla_\Gamma f({\bf x}^i)\|_2 &\ge& c_H (\| \nabla_\Phi f({\bf x}^i)- \nabla_\Phi f({\bf x}) \|_2  - \|\nabla_\Phi f({\bf x})\|_2 )\nonumber\\
&\ge& \frac{c_H  (1 - \delta)}{\xi}   \|{\bf r}^i\|_2 - c_H \|\nabla_I f({\bf x})\|_2, \nonumber
\end{eqnarray} \end{small}

\vspace{-3mm}
\noindent where the last inequality follows from Lemma~\ref{lemma:twoinequalities}. 
The component $\|\nabla_\Gamma f({\bf x}^i)\|_2$ can also be upper bounded as
\vspace{-1mm}
\begin{small}
\begin{eqnarray}
\|\nabla_\Gamma f({\bf x}^i)\|_2 &\le&\frac{1}{\xi} \|\xi \nabla_\Gamma f({\bf x}^i)- \xi\nabla_\Gamma f({\bf x})\|_2 + \|\nabla_\Gamma f({\bf x})\|_2 \nonumber \\
&\le&  \frac{1}{\xi}  \|\xi \nabla_\Gamma f({\bf x}^i) -  \xi \nabla_\Gamma f({\bf x}) - {\bf r}^i_\Gamma + {\bf r}^i_\Gamma\|_2 +  \|\nabla_\Gamma f({\bf x})\|_2 \nonumber\\
&\le &  \frac{1}{\xi}  \| \xi \nabla_{\Gamma\cup \Omega} f({\bf x}^i) -  \xi \nabla_{\Gamma\cup \Omega} f({\bf x}) - {\bf r}^i_{\Gamma\cup \Omega}\|_2 + \|{\bf r}^i_\Gamma\|_2 + \|\nabla_\Gamma f({\bf x})\|_2 \nonumber\\
&\le&\frac{\delta}{\xi} \cdot  \|{\bf r}^i\|_2 + \frac{1}{\xi}\|{\bf r}^i_\Gamma\|_2+ \|\nabla_{I} f({\bf x})\|_2, \nonumber
\end{eqnarray}
\end{small}

\vspace{-3mm}
\noindent where the last inequality follows from condition $(\xi, \delta, \mathbb{M}(8k, g))$-WRSC and the fact that ${\bf r}^i_{\Gamma\cup \Omega} = {\bf r}^i$. Let  $\eta = \left(c_H \cdot (1 - \delta) - \delta\right)$. Combining the two bounds and grouping terms, we have \begin{small}$\|{\bf r}^i_\Gamma\|  \ge \eta \|{\bf r}^i\|_2 - \xi(1+c_H) \|\nabla_I f({\bf x})\|_2
$\end{small}. 
After a number of algebraic manipulations similar to those used in~\cite{hegde2014approximation} Page 11, we prove the lemma. 
\end{proof}

\begin{theorem}Consider the graph-structured sparsity model $\mathbb{M}(k, g)$ for some $k, g\in \mathbb{N}$ and a cost function $f: \mathbb{R}^n \rightarrow \mathbb{R}$ that satisfies condition $\left(\xi, \delta, \mathbb{M}(8k, g)\right)$-WRSC. If $\eta = c_H(1 - \delta) - \delta > 0$, then for any true ${\bf x} \in \mathbb{R}^n$ with $\text{supp}({\bf x}) \in \mathbb{M}(k, g)$, the iterates of Algorithm 1 obey
\begin{eqnarray}
\|{\bf x}^{i+1}-{\bf x}\|_2 \le \alpha \|{\bf x}^i-{\bf x}\|_2 + \beta \|\nabla_I f({\bf x})\|,\label{decay-rate-1}
\end{eqnarray}
\noindent where $\beta = \frac{\xi(1+c_T)}{1-\delta}  \left[\frac{(1 + c_H)}{\eta} + \frac{\eta (1 + c_H)}{\sqrt{1 - \eta^2}} + 1\right]$, $\alpha = \frac{(1+c_T)}{1-\delta}   \sqrt{1 - \eta^2}$, and $I = \arg \max_{S \in \mathbb{M}(8k, g)} \|\nabla_{S} f({\bf x})\|_2.$
\label{theorem:convergence}
\end{theorem}

\begin{proof}
Let ${\bf r}^{i+1} = {\bf x}^{i+1} - {\bf x}$. $\|{\bf r}^{i+1}\|_2$ is upper bounded as 
\begin{eqnarray}
\|{\bf r}^{i+1}\|_2 = \|{\bf x}^{i+1} - {\bf x}\|_2 
&\le & \|{\bf x}^{i+1} - {\bf b}\|_2  + \|{\bf x} - {\bf b}\|_2 \nonumber\\
&\le& c_T \|{\bf x} - {\bf b}\|_2  + \|{\bf x} - {\bf b}\|_2 \nonumber\\
&=& (1 + c_T) \|{\bf x} - {\bf b}\|_2, \nonumber
\end{eqnarray}
\noindent which follows from the definition of tail approximation. 
The component $\|({\bf x} - {\bf b})_\Omega\|_2^2$ is upper bounded as 
\begin{eqnarray}
\|({\bf x} - {\bf b})_\Omega\|_2^2 & =& \langle {\bf b} - {\bf x}, ({\bf b} - {\bf x})_\Omega \rangle \nonumber \\ 
& =& \langle {\bf b} - {\bf x} - \xi \nabla_\Omega f({\bf b}) + \xi \nabla_\Omega f({\bf x}), ({\bf b} - {\bf x})_\Omega\rangle -  \langle\xi \nabla_\Omega f({\bf x}), ({\bf b} - {\bf x})_\Omega \rangle \nonumber \\
 & \le& \delta \|{\bf b} - {\bf x}\|_2 \|({\bf b} - {\bf x})_\Omega\| + \xi \|\nabla_\Omega f({\bf x})\|_2 \|({\bf b} - {\bf x})_\Omega\|_2 \nonumber,
\end{eqnarray}
\noindent where the second equality follows from the fact that $\nabla_\Omega f({\bf b}) = 0$ since ${\bf b}$ is the solution to the problem in Step 6 of Algorithm 1, and the last inequality follows from condition $(\xi, \delta, \mathbb{M}(8k, g))$-WRSC. 
After simplification, we have 
\begin{equation}
\|({\bf x} - {\bf b})_\Omega\|_2 \le \delta \|{\bf b} - {\bf x}\|_2 + \xi \|\nabla_\Omega f({\bf x})\|_2.
\end{equation}
\noindent It follows that
\begin{eqnarray}
\|({\bf x} - {\bf b})\|_2 \le \|({\bf x} - {\bf b})_\Omega\|_2 + \|({\bf x} - {\bf b})_{\Omega^c}\|_2 \le \delta \|{\bf b} - {\bf x}\|_2 + \xi \|\nabla_\Omega f({\bf x})\|_2 + \|({\bf x} - {\bf b})_{\Omega^c}\|_2. \nonumber 
\end{eqnarray}
\noindent After rearrangement we obtain 
\begin{eqnarray}
\|{\bf b} - {\bf x}\|_2 & \le& \frac{\|({\bf b} - {\bf x})_{\Omega^c}\|_2}{1 - \delta} + \frac{\xi \|\nabla_\Omega f({\bf x})\|_2}{1 - \delta} \nonumber \\
& =& \frac{\|{\bf x}_{\Omega^c}\|_2}{1 - \delta} + \frac{\xi \|\nabla_\Omega f({\bf x})\|_2}{1 - \delta} 
= \frac{\|({\bf x} - {\bf x}^i)_{\Omega^c}\|_2}{1 - \delta} + \frac{\xi \|\nabla_\Omega f({\bf x})\|_2}{1 - \delta} \nonumber \\
& =& \frac{\|{\bf r}_{\Omega^c}^{i}\|_2}{1 - \delta} + \frac{\xi \|\nabla_\Omega f({\bf x})\|_2}{1 - \delta} 
\le \frac{\|{\bf r}^i_{\Gamma^c}\|_2}{1 - \delta} + \frac{\xi \|\nabla_\Omega f({\bf x})\|_2}{1 - \delta}, \nonumber 
\end{eqnarray}
\noindent where the first equality follows from the fact that $\text{supp}({\bf b}) \subseteq \Omega$, the second and last inequalities follow from the fact that $\Omega = \Gamma \cup \text{supp}({\bf x}^i)$. Combining above inequalities, we obtain 
\begin{eqnarray}
\|{\bf r}^{i+1}\|_2 \le (1 + c_T) \frac{\|{\bf r}^i_{\Gamma^c}\|_2}{1-\delta} + (1 + c_T) \frac{\xi \|\nabla_I f({\bf x})\|_2}{1-\delta}. \nonumber
\end{eqnarray}
\noindent From Lemma~\ref{lemma:r-Complement}, we have 
\begin{eqnarray}
\|{\bf r}^i_{\Gamma^c}\|_2 \le \sqrt{1 - \eta^2} \|{\bf r}^i\|_2 + \left[\frac{\xi(1 + c_H)}{\eta} + \frac{\xi \eta (1 + c_H)}{\sqrt{1 - \eta^2}}\right] \|\nabla_I f({\bf x})\|_2 \nonumber
\end{eqnarray}
\noindent Combining the above inequalities, we prove the theorem. 
\end{proof}

As indicated in Theorem~\ref{theorem:convergence}, under proper conditions the estimator error of \textsc{Graph}-\textsc{Mp} is determined by the multiplier of $\|\nabla_{S} f({\bf x})\|_2$, and the convergence rate before reaching this error level is geometric. In particular, if the true ${\bf x}$ is sufficiently close to an unconstrained minimum of $f$, then the estimation error is negligible because $\|\nabla_{S} f({\bf x})\|_2$ has a small magnitude. Especially, in the ideal case where $\nabla f({\bf x}) = 0$, it is guaranteed that we can obtain the true ${\bf x}$ to arbitrary precision. If we further assume that $\alpha = \frac{(1+c_T) \sqrt{1 - \eta^2}}{\sqrt{1-\delta}} < 1$, then exact recovery can be achieved in finite iterations. 

The shrinkage rate $\alpha < 1$ controls the convergence of \textsc{Graph}-\textsc{Mp}, and it implies that when $\delta$ is very small, the approximation factors $c_H$ and $c_T$  satisfy
\begin{eqnarray}
c^2_H > 1 - 1/ (1+c_T)^2.
\end{eqnarray}
We note that the head and tail approximation algorithms designed in \cite{hegde2015nearly} do not satisfy the above condition, with $c_T = \sqrt{7}$ and  $c_H = \sqrt{1/14}$. However, as proved in~\cite{hegde2015nearly}, the approximation factor $c_H$ of any given head approximation algorithm can be \textbf{boosted} to any arbitrary constant $c_H^\prime < 1$, such that the above condition is satisfied. Empirically it is not necessary to “boost” the
head-approximation algorithm as strongly as suggested by the analysis in~\cite{hegde2014approximation}. 

\begin{theorem}
Let ${\bf x} \in \mathbb{R}^n$ be a true optimum such that $\text{supp}({\bf x}) \in \mathbb{M}(k, g)$, and $f: \mathbb{R}^n \rightarrow \mathbb{R}$ be a  cost function that satisfies condition $\left(\xi, \delta, \mathbb{M}(8k, g)\right)$-WRSC. Assuming that $\alpha < 1$, \textsc{Graph}-\textsc{Mp} returns a $\hat{\bf x}$ such that,  $\text{supp}(\hat{\bf x}) \in \mathbb{M}(5k, g)$ and $\|{\bf x} - \hat{\bf x}\|_2 \le c \|\nabla_I f({\bf x})\|_2$, where $c= (1+\frac{\beta}{1-\alpha})$ is a fixed constant. Moreover, \textsc{Graph}-\textsc{Mp} runs in time 
\begin{eqnarray}
O\left((T+|\mathbb{E}|\log^3 n) \log (\|{\bf x}\|_2/ \|\nabla_I f({\bf x})\|_2)\right) \label{timecomplexity-0},
\end{eqnarray}
where $T$ is the time complexity of one execution of the subproblem in Line 6. In particular, if $T$ scales linearly with $n$, then  \textsc{Graph}-\textsc{Mp} scales nearly linearly with $n$. 
\end{theorem}
\begin{proof}
The i-th iterate of Algorithm 1 satisfies
\begin{eqnarray}
\|{\bf x} - {\bf x}^i\|_2 \le \alpha^i \|{\bf x}\|_2 + \frac{\beta}{1-\alpha} \|\nabla_I f({\bf x})\|_2.
\end{eqnarray}
\noindent After $t = \left\lceil \log \left(\frac{\|{\bf x}\|_2}{\|\nabla_I f({\bf x})\|_2}\right) / \log \frac{1}{\alpha} \right\rceil $ iterations,
Algorithm 1 returns an estimate $\hat{x}$ satisfying $\|{\bf x} - \hat{{\bf x}}\|_2 \le (1 + \frac{\beta}{1 - \alpha}) \|\nabla_I f({\bf x})\|_2.$ The time complexities of both head and tail approximations are $O(|\mathbb{E}| \log^3 n)$. The time complexity of one iteration in Algorithm 1 is $(T+|\mathbb{E}|\log^3 n)$, and the total number of iterations is $\left \lceil \log \left(\frac{\|{\bf x}\|_2}{\|\nabla_I f({\bf x})\|_2}\right) / \log \frac{1}{\alpha} \right \rceil $, and the overall time complexity follows. 
\end{proof}

\begin{remark}
The previous algorithm \textsc{Graph}-\textsc{Cosamp}~\cite{hegde2015nearly} for compressive sensing is a special case of \textsc{Graph}-\textsc{Mp}. Assume $f({\bf x}) = \|{\bf y} - {\bf A}{\bf x}\|_2^2$. 1) \textbf{Reduction.} The gradient in Step 3 of \textsc{Graph}-\textsc{Mp} has the form: $\nabla f({\bf x}^i) = -{\bf A}^T({\bf y}-{\bf A}{\bf x}^i)$, and an analytical form of ${\bf b}$ in Step 6 can be obtained as: ${\bf b}_\Omega = {\bf A}^+_\Omega {\bf y}$ and ${\bf b}_{\Omega^c} = 0$, where ${\bf A}^+ = {\bf A}^T({\bf A}^T{\bf A})^{-1}$, which indicates that  \textsc{Graph}-\textsc{Mp} reduces to \textsc{Graph}-\textsc{Cosamp} in this scenario. 2) \textbf{Shrinkage rate.} The shrinkage rate $\alpha$ of \textsc{Graph}-\textsc{Mp} is analogous to that of \textsc{Graph}-\textsc{Cosamp}, even though that the shrinkage rate of \textsc{Graph}-\textsc{Cosamp} is optimized based on the $RIP$ sufficient constants. In particular, they are identical when $\delta$ is very small. 3) \textbf{Constant component.}
Assume that $\xi = 1$. Condition $(\xi, \delta, \mathbb{M}(k, g))$-WRSC then reduces to the RIP condition in compressive sensing. 
Let ${\bf e} = {\bf y}-{\bf A}{\bf x}$. The component $\|\nabla f({\bf x}^i)\|_2 = \|{\bf A}^T {\bf e}\|_2$ is upper bounded by $\sqrt{1 + \delta} \|{\bf e}\|_2$~\cite{hegde2014approximation}. The constant $ \beta \|\nabla_I f({\bf x})\|$ is then upper bounded by $\frac{\xi(1+c_T)\sqrt{1 + \delta}}{1-\delta}  \left[\frac{(1 + c_H)}{\eta} + \frac{\eta (1 + c_H)}{\sqrt{1 - \eta^2}} + 1\right]\|{\bf e}\|_2$ that is analogous to the constant of \textsc{Graph}-\textsc{Cosamp}, and they are identical when $\delta$ is very small. 
\end{remark}

\section{Application in Graph Scan Statistic Models}
In this section, we specialize \textsc{Graph}-\textsc{Mp} to optimize a number of graph scan statistic models for the task of connected subgraph detection. Given a graph $\mathbb{G} = (\mathbb{V}, \mathbb{E})$, where $\mathbb{V} = [n]$, $\mathbb{E} \subseteq \mathbb{V} \times \mathbb{V}$, and each node $v$ is associated with a vector of features ${\bf c}(v) \in \mathbb{R}^p$.  Let $S\subseteq \mathbb{V}$ be a connected subset of nodes. 
A graph scan statistic, $F(S) = \log \frac{\text{Prob}(\text{Data} | H_1(S))}{\text{Prob}(\text{Data} | H_0)}$, corresponds to the generalized likelihood ratio test (GLRT) to verify the null hypothesis ($H_0$): ${\bf c}(v) \sim \mathcal{D}_1, \forall v \in \mathbb{V}$, where $\mathcal{D}_1$ refers to a predefined background distribution, against the alternative hypothesis ($H_1(S)$): ${\bf c}(v) \sim \mathcal{D}_2, \forall v \in S$ and ${\bf c}(v) \sim \mathcal{D}_1, \forall v \in \mathbb{V} \setminus S$, where $\mathcal{D}_2$  refers to a predefined signal distribution. 
%Each specific graph scan statistic corresponds to a specific pair of background distribution $\mathcal{D}_1$ and signal distribution $\mathcal{D}_2$, such as Kulldorff's scan statistic, elevated mean scan statistic, and various others. 
The detection problem is formulated as
\begin{eqnarray}
\min_{S \subseteq \mathbb{V}} -F(S)\ \ \ s.t.\ \ \ |S| \le k \text{ and  } S \text{ is connected},
\end{eqnarray}
\noindent where $k$ is a predefined bound on the size of $S$. 

Taking elevated mean scan (EMS) statistic for instance, it aims to decide between $H_0: {\bf c}(v) \sim \mathcal{N}(0, 1), \forall v \in \mathbb{V}$ and $H_1(S)$: ${\bf c}(v) \sim \mathcal{N}(\mu, 1), \forall v \in S$ and ${\bf c}(v) \sim \mathcal{N}(0, 1), \forall v \in \mathbb{V} \setminus S$, where for simplicity each node $v$ only has a univariate feature $c(v) \in \mathbb{R}$. 
This statistic is popularly used for detecting signals among node-level numerical features on graph~\cite{qian2014connected,arias2011detection} and is formulated as $F(S) = (\sum_{v \in S} c(v))^2 / |S|$. Let the vector form of $S$ be $x\in \{0, 1\}^n$, such that $\text{supp}(x) = S$. The connected subgraph detection problem can be reformulated as
\begin{eqnarray}
\min_{x \in \{0, 1\}^n} - \frac{({\bf c}^T {\bf x})^2}{({\bf 1}^T {\bf x})}  \ \  \ s.t.\ \  \ \text{supp}({\bf x}) \in \mathbb{M}(k, g = 1), \label{emss}
\end{eqnarray}
\noindent where ${\bf c} = [c(1), \cdots, c(n)]^T$. To apply \textsc{Graph}-\textsc{Mp}, we relax the input domain of ${\bf x}$ such that ${x \in [0, 1]^n}$, and the connected subset of nodes can be found as $S=\text{supp}({\bf x}^\star)$, the support set of the estimate ${\bf x}^\star$ that minimizes the strongly convex function~\cite{bach2011learning}: 
\begin{eqnarray}
\min_{x \in \mathbb{R}^n} f({\bf x}) = - \frac{({\bf c}^T {\bf x})^2}{({\bf 1}^T {\bf x})} + \frac{1}{2} {\bf x}^T {\bf x} \  \ s.t.\  \text{supp}({\bf x}) \in \mathbb{M}(k, 1).\nonumber 
\end{eqnarray}
Assume that ${\bf c}$ is normalized, and hence $0 \le c_i < 1, \forall i$. Let $\hat{c} = \max \{c_i\}$.  The Hessian matrix of the above objective function $\nabla^2 f({\bf x})\succ 0$ and satisfies the inequalities:
\begin{eqnarray}
 (1 - \hat{c}^2) \cdot  \textbf{I}\preceq \textbf{I} - ({\bf c} - \frac{{\bf c}^T {\bf x}}{{\bf 1}^T {\bf x}} {\bf 1} ) ({\bf c} - \frac{{\bf c}^T {\bf x}}{{\bf 1}^T {\bf x}} {\bf 1})^T \preceq 1\cdot \textbf{I}. 
\end{eqnarray}
\noindent According to Lemma 1 (b) in \cite{yuan2014icml}),  the objective function $f({\bf x})$ satisfies condition $(\xi, \delta, \mathbb{M}(8k, g))$-WRSC that $\delta = \sqrt{1 - 2 \xi (1 - \hat{c}^2) + \xi^2},$ for any $\xi$ such that $\xi < 2 (1 - \hat{c}^2)$. Hence, the geometric convergence of  \textsc{Graph}-\textsc{Mp} as shown in Theorem~\ref{theorem:convergence} is guaranteed. We note that not all the graph scan statistic functions satisfy the WRSC condition, but, as shown in our experiments, \textsc{Graph}-\textsc{Mp} works empirically well for all the scan statistic functions tested, and the maximum number of iterations to convergence for optimizing each of these scan statistic functions was less than 10. 

We note that our proposed method \textsc{Graph}-\textsc{Mp} is also applicable to general sparse learning problems (e.g. sparse logistic regression, sparse principle component analysis) subject to graph-structured constraints, and to a variety of subgraph detection problems, such as the detection of anomalous subgraphs, bursty subgraphs, heaviest subgraphs, frequent subgraphs or communication motifs, predictive subgraphs, and compression subgraphs.

\section{Experiments}
This section evaluates the effectiveness and efficiency of the proposed \textsc{Graph}-\textsc{Mp} approach for  connected subgraph detection. The implementation of \textsc{Graph}-\textsc{Mp} is available at https://github.com/baojianzhou/Graph-MP. 
\subsection{Experiment Design}

{\bf\ \ \ \ Datasets: } 1) \textbf{Water Pollution Dataset.} The Battle of the
Water Sensor Networks (BWSN)~\cite{ostfeld2008battle} provides a real-world
network of 12,527 nodes and 14831 edges, and 4 nodes
with chemical contaminant plumes that are distributed in
four different areas. The spreads of contaminant plumes
 were simulated using the water network simulator
EPANET for 8 hours.
For each hour, each node has a sensor that reports 1 if it
is polluted; 0, otherwise. We randomly selected $K$
percent nodes, and flipped their sensor binary values in order to test the robustness of methods
to noises, where
$K \in \{2, 4, 6, 8, 10\}$. \textbf{The objective is to detect the set of polluted nodes}. 2) \textbf{High-Energy Physics Citation Network.} The CitHepPh (high energy physics phenomenology) citation graph is from the e-print arXiv and covers all the citations within a dataset of 34,546 papers with 421,578 edges during the period from January 1993 to April 2002. Each paper is considered as a node, each citation is considered as a edge (direction is not considered), and each node has one attribute denoting the number of citations in a specific year ($t = 1993, \cdots, t = 2002$), and another attribute denoting the average number of citations in that year. \textbf{The objective is to detect a connected subgraph of nodes (papers) whose citations are abnormally high in comparison with the citations of nodes outside the subgraph.} This subgraph is considered an indicator of a potential emerging research area. The data before 1999 is considered as training data, and the data from 1999 to 2002 is considered as testing data.

\begin{table*}[t]
\small
\footnotesize
%\scriptsize
\centering

\begin{tabular}{|c | c | c | c | c | c |c | c | c |}
\hline
 &\multicolumn{4}{c|}{WaterNetwork}&\multicolumn{4}{c|}{CitHepPh}\\
\cline{2-9}
  & Kulldorff & EMS & EBP & \shortstack{Run Time \\ (sec.)} & Kulldorff & EMS & EBP & Run Time\\
\hline
Our Method   &\textbf{1668.14} & 499.97 & \textbf{4032.48} & 40.98 & \textbf{13859.12} & \textbf{142656.84} & \textbf{9494.62} & 97.21\\
\texttt{GenFusedLasso}  & 541.49 & 388.04 & 3051.22 & 901.51 & 2861.6 & 60952.57 & 6472.84 & 947.07 \\
\texttt{EdgeLasso} & 212.54 & 308.11 & 1096.17 & 70.06 & 39.42 & 2.0675.89 & 261.71 & 775.61\\
\texttt{GraphLaplacian} & 272.25 & 182.95 & 928.41 & 228.45 & 1361.91 & 29463.52 & 876.31 & 2637.65\\
\texttt{LTSS} & 686.78 & 479.40 & 1733.11 & \textbf{1.33} & 11965.14 & 137657.99 & 9098.22 & \textbf{6.93}\\
\texttt{EventTree} & 1304.4 & 744.45 & 3677.58 & 99.27 & 10651.23 & 127362.57 & 8295.47 & 100.93\\
\texttt{AdditiveGraphScan} & 1127.02 & \textbf{761.08} & 2794.66 & 1985.32 & 12567.29 & 140078.86 & 9282.28 & 2882.74\\
\texttt{DepthFirstGraphScan} & 1059.10 & 725.65 & 2674.14 & 8883.56 & 7148.46 & 62774.57 & 4171.47 & 9905.45\\
\texttt{NPHGS} & 686.78 & 479.40 & 1733.11 & 1339.46 & 12021.85 & 137963.5 & 9118.96 & 1244.80\\
\hline
\end{tabular}
%\vspace{-3mm}
\caption{Comparison on scores of the three graph scan statistics based on connected subgraphs returned by comparison methods. EMS and EBP refer to Elevated Mean Scan Statistic and Expectation-Based Poisson Statistic, respectively.}
\label{table:comparison}
%\vspace{-5mm}
\end{table*}

{\bf Graph Scan Statistics: } Three graph scan statistics were considered, including Kulldorff's scan statistic~\cite{neill2012fast}, expectation-based Poisson scan statistic (EBP)~\cite{neill2012fast}, and elevated mean scan statistic (EMS, Equation (\ref{emss}))~\cite{qian2014connected}. 
The first two require that each node has an observed count of events at that node, and an expected count. For the water network dataset, the report of the sensor (0 or 1) at each node is considered as the observed count, and the noise ratio is considered as the expected count. For the CiteHepPh dataset, the number of citations is considered as the observed count, and the average number of citations is considered as the expected count. For the EMS statistic, we consider the ratio of observed and expected counts as the feature. 

{\bf Comparison Methods:} Seven state-of-the-art baseline methods are considered, including \texttt{EdgeLasso}~\cite{sharpnack2012sparsistency}, \texttt{GraphLaplacian}~\cite{sharpnack2012changepoint}, LinearTimeSubsetScan (\texttt{LTSS})~\cite{neill2012fast}, \texttt{EventTree}~\cite{RozenshteinAGT14}, \texttt{AdditiveGraphScan}~\cite{conf/icdm/SpeakmanZN13}, \texttt{DepthFirstGraphScan}~\cite{Speakman-14}, and \texttt{NPHGS}~\cite{DBLP:conf/kdd/ChenN14}. We
 followed strategies recommended by authors in their original
papers to tune the related model parameters. Specifically, for EventTree and Graph-Laplacian, we tested the set of $\lambda$ values: $\{ 0.02, 0.04, \cdots, 2.0\}$. \texttt{DepthFirstScan} is an exact search algorithm and has an exponential time cost in the worst case scenario. We hence set a constraint on the depth of its search to 10 in order to reduce its time complexity. 

%We note that the algorithm based on linear matrix inequalities (LMI)~\cite{qian2014connected} is not scalable to large graphs (e.g., more than 1000 nodes) and was hence not considered in the experiments. 

We also implemented the generalized fused lasso model (\texttt{GenFusedLasso}) for these three graph scan statistics  using the framework of alternating direction method of multipliers (ADMM). \texttt{GenFusedLasso} is formalized as 
\begin{eqnarray}
\min_{{\bf x} \in \mathbb{R}^n} -f({\bf x}) + \lambda \sum\nolimits_{(i, j) \in \mathbb{E}} \|x_i - x_j\|,
\end{eqnarray}
\noindent where $f({\bf x})$ is a predefined graph scan statistic and the trade-off parameter $\lambda$ controls the degree of smoothness of neighboring entries in ${\bf x}$. We applied the heuristic rounding step proposed in~\cite{qian2014connected} to the numerical vector ${\bf x}$ to identify the connected subgraph. We tested the $\lambda$ values: $\{0.02, 0.04, \cdots, 2.0, 5.0, 10.0\}$ and returned the best result. 

{\bf Our Proposed Method \textsc{Graph}-\textsc{Mp}: } Our proposed \textsc{Graph}-\textsc{Mp} has a single parameter $k$, an upper bound of the subgraph size. We set $k = 1000$ by default, as the sizes of subgraphs of interest are often small; otherwise, the detection problem could be less challenging. We note that, to obtain the best performance of our proposed method \textsc{Graph}-\textsc{Mp}, we should try a set of different $k$ values ($k=50, 100, 200, 300, \cdots, 1000$) and return the best.  

{\bf Performance Metrics: } The overall scores of the three graph scan statistics of the connected subgraphs returned by the competitive methods were compared and analyzed. The objective is to identify methods that could find the connected subgraphs with the largest scores. The running times of different methods are compared.

\begin{figure*}[t]
  \centering
  \subfigure[WaterNetwork(Kulldorff)]{
    \includegraphics[width=0.35\textwidth]{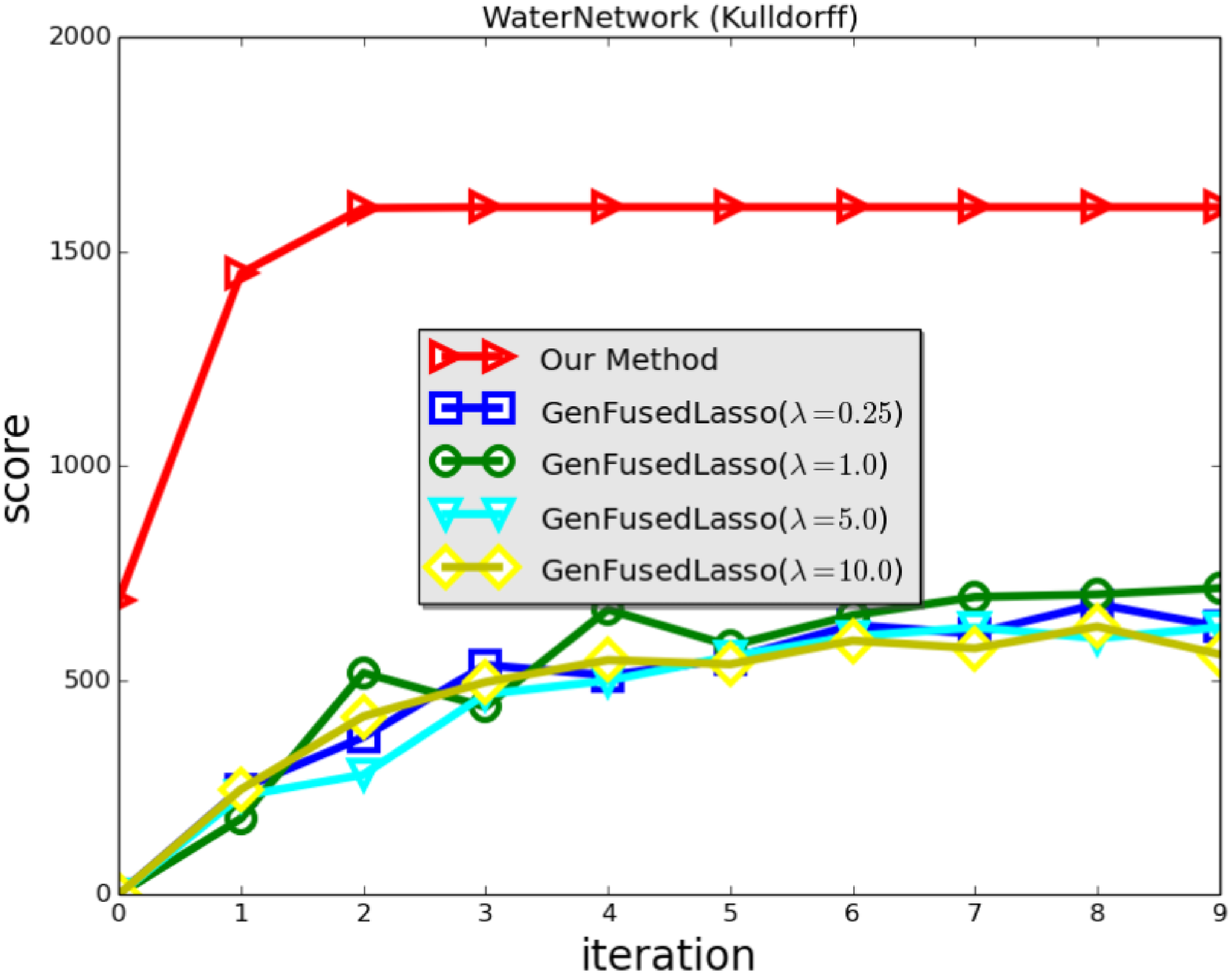}}\hspace{7em}
  \subfigure[CitHepPh(EMS)]{
    \includegraphics[width=0.35\textwidth]{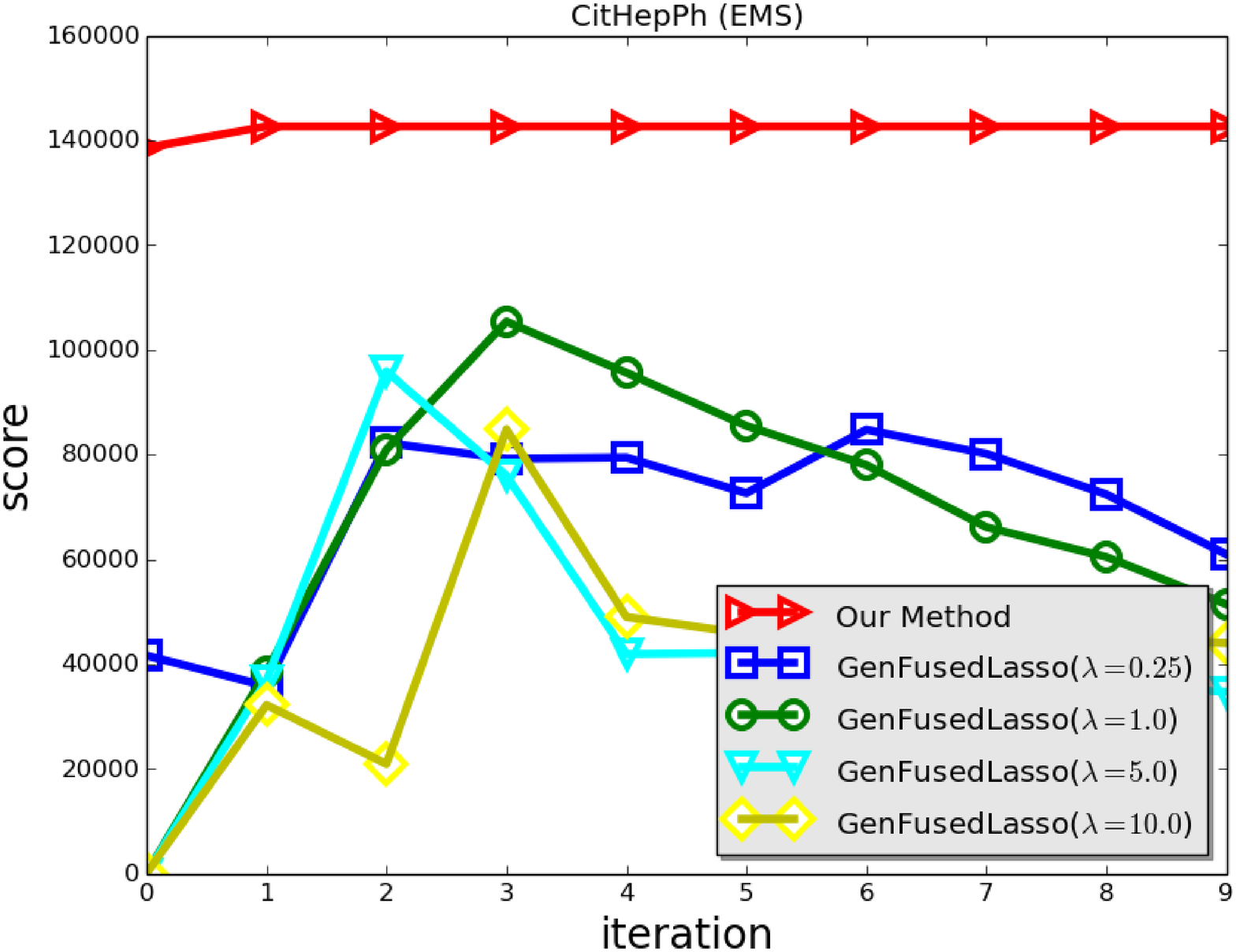}}
%\vspace{-4mm}
  \caption{Evolving curves of graph scan statistic scores between our method and \texttt{GenFusedLasso}.}
  \label{fig:comparsion-of-iterations} %% label for entire figure
  %\vspace{-2mm}
\end{figure*}

\subsection{Evolving Curves of Graph Scan Statistics}

Figure~\ref{fig:comparsion-of-iterations} presents the comparison between our method and \texttt{GenFusedLasso} on the scores of the best connected subgraphs that are identified at different iterations based on the Kulldorff's scan statistic and the EMS statistic. Note that, a heuristic rounding process as proposed in~\cite{qian2014connected} was applied to the numerical vector ${\bf x}^i$ estimated by \texttt{GenFusedLasso}  in order to identify the best connected subgraph at each iteration $i$. As the setting of the parameter $\lambda$ will influence the quality of the detected connected subgraph, the results based on  different $\lambda$ values are also shown in Figure~\ref{fig:comparsion-of-iterations}.  We observe that our proposed method \textsc{Graph}-\textsc{Mp} converged in less than 5 steps and the qualities (scan statistic scores) of the connected subgraphs identified \textsc{Graph}-\textsc{Mp} at different iterations were consistently higher than those returned by \texttt{GenFusedLasso}. 

\subsection{Comparison on Optimization Quality}
The comparison between our method and the other eight baseline methods is shown in Table~\ref{table:comparison}. The scores of the three graph scan statistics of the connected subgraphs returned by these methods are reported in this table. The results in indicate that our method outperformed all the baseline methods on the scores, except that \texttt{AdditiveGraphScan} achieved the highest EMS score (761.08) on the water network data set. Although \texttt{AdditiveGraphScan} performed reasonably well in overall, this algorithm is a heuristic algorithm and does not have theoretical guarantees. 

\subsection{Comparison on Time Cost}
Table~\ref{table:comparison} shows the time costs of all competitive methods on the two benchmark data sets. The results indicate that our method was the second fastest algorithm over all the comparison methods. In particular, the running times of our method were 10+ times faster than the majority of the methods. 

\section{Conclusion and Future Work}
This paper presents, \textsc{Graph}-\textsc{Mp}, an efficient algorithm to minimize a general nonlinear function subject to graph-structured sparsity constraints. For the future work, we plan to explore graph-structured constraints other than connected subgraphs, and analyze  theoretical properties of  \textsc{Graph}-\textsc{Mp} for cost functions that do not satisfy the WRSC condition. 

\section{Acknowledgements}
This work is supported by the Intelligence Advanced
Research Projects Activity (IARPA) via Department of
Interior National Business Center (DoI/NBC) contract
D12PC00337. The views and conclusions contained herein
are those of the authors and should not be interpreted as necessarily
representing the official policies or endorsements,
either expressed or implied, of IARPA, DoI/NBC, or the US
government.

\bibliographystyle{ijcai16}

\end{document}